\documentclass[11pt]{article}

\usepackage[utf8]{inputenc} 
\usepackage[T1]{fontenc}    
\usepackage{hyperref}       
\usepackage{url}            
\usepackage{booktabs}       
\usepackage{amsfonts}       
\usepackage{nicefrac}       
\usepackage{microtype}      
\usepackage{graphicx}
\usepackage[normalem]{ulem}

\usepackage{amssymb}
\usepackage{amsmath}
\usepackage{natbib}
\usepackage{fullpage}

\usepackage{algorithm}
\usepackage{algorithmic}
\usepackage{amsthm}
\theoremstyle{plain}
\newtheorem{theorem}{Theorem}
\newtheorem{proposition}{Proposition}
\newtheorem{lemma}{Lemma}
\newtheorem{corollary}{Corollary}
\theoremstyle{definition}
\newtheorem{definition}{Definition}

\theoremstyle{remark}
\newtheorem{remark}{Remark}

\newcommand{\cH}{\mathcal{H}}

\newcommand{\cX}{\mathcal{X}}
\newcommand{\cP}{\mathcal{P}}

\newcommand{\CR}{\mathrm{CR}}
\newcommand{\opt}{F^*}
\newcommand{\alphab}{\bar{\alpha}}
\newcommand{\psib}{\bar{\psi}}
\newcommand{\betab}{\bar{\beta}}
\newcommand{\hb}{\bar{h}}
\DeclareMathOperator*{\argmin}{arg\,min}
\DeclareMathOperator*{\argmax}{arg\,max}
\DeclareMathOperator*{\E}{\mathbf{E}}

\newcommand{\re}{\mathbb{R}}

\newcommand{\linner}{\left\langle}
\newcommand{\rinner}{\right\rangle}
\renewcommand{\epsilon}{\varepsilon}

\title{Adaptive Learning Rate for Follow-the-Regularized-Leader: \\ Competitive Analysis and Best-of-Both-Worlds}

\author{
  Shinji Ito\footnote{
    NEC Corporation and RIKEN; \texttt{i-shinji@nec.com}.
  }
  \and
  Taira Tsuchiya\footnote{
    The University of Tokyo; 
    \texttt{tsuchiya@mist.i.u-tokyo.ac.jp}.
  }
  \and
  Junya Honda\footnote{
    Kyoto University and RIKEN;
    \texttt{honda@i.kyoto-u.ac.jp}.
  }
}

\begin{document}
\maketitle

\begin{abstract}
Follow-The-Regularized-Leader (FTRL) is known as an effective and versatile approach in online learning, where appropriate choice of the learning rate is crucial for smaller regret. To this end, we formulate the problem of adjusting FTRL's learning rate as a sequential decision-making problem and introduce the framework of competitive analysis. We establish a lower bound for the competitive ratio and propose update rules for learning rate that achieves an upper bound within a constant factor of this lower bound. Specifically, we illustrate that the optimal competitive ratio is characterized by the (approximate) monotonicity of components of the penalty term, showing that a constant competitive ratio is achievable if the components of the penalty term form a monotonically non-increasing sequence, and derive a tight competitive ratio when penalty terms are $\xi$-approximately monotone non-increasing. Our proposed update rule, referred to as \textit{stability-penalty matching}, also facilitates constructing the Best-Of-Both-Worlds (BOBW) algorithms for stochastic and adversarial environments. In these environments our result contributes to achieve tighter regret bound and broaden the applicability of algorithms for various settings such as multi-armed bandits, graph bandits, linear bandits, and contextual bandits.
\end{abstract}

\section{Introduction}
In the research field of online learning and bandit algorithms,
the \textit{follow-the-regularized-leader} (FTRL) framework 
offers a promising approach to achieving sublinear regret.
In this framework,
we choose an action $a_t$ in each round $t$,
on the basis of $x_t \in \cX$,
a solution to the following convex optimization problem:
\begin{align}
	\label{eq:defFTRLx}
	x_t \in \argmin_{ x \in \cX }
	\left\{
		\sum_{s=1}^{t-1} \hat{f}_s(x)
		+
		\frac{1}{\eta_t}
		\psi(x)
	\right\},
\end{align}
where $\cX$ is a convex set,
$\{ \hat{f}_s \}$ are estimators or surrogates of the loss functions,
$\{ \eta_t \}$ are learning rate parameters
that are positive and monotone non-decreasing,
and $\psi$ is a convex regularizer function.
This approach can be interpreted as a comprehensive framework that includes Online Gradient Descent~\citep{zinkevich2003online} and the Hedge algorithm~\citep{littlestone1994weighted,arora2012multiplicative,freund1997decision},
which demonstrates its effectiveness across various online learning and bandit problems, such as
multi-armed bandits~\citep{auer2002nonstochastic},
linear bandits~\citep{abernethy2008competing,cesa2012combinatorial},
and episodic MDPs~\citep{lee2020bias}.

To harness the effectiveness of FTRL, it is crucial to appropriately set the learning rate.
Here, a fixed learning rate determined by time horizon $T$ often suffices when $T$ is predefined and the goal is the worst-case optimality.
On the other hand,
adaptive update of the learning rate based on feedback received at each time step has been considered
when $T$ is not predetermined and/or the goal is to achieve the optimality beyond the worst case with better practical performance.
Such methods of adaptive learning rate
have been shown to be beneficial in achieving data-dependent bounds
\citep{cesa2007improved,orabona2015scale,erven2011adaptive}
and in constructing \textit{best-of-both-worlds} (BOBW) algorithms
\citep{gaillard2014second,ito2021parameter,jin2023improved}
that attain (nearly) optimal performance in both adversarial and stochastic settings.
Other literature on adaptive learning rates is also mentioned in Appendix~\ref{sec:related}.

This paper aims to develop a generic methodology for sequentially adjusting learning rate
in FTRL,
and to investigate its limitations.
A standard analysis for FTRL (e.g., in \citealp[Exercise 28.12]{lattimore2020bandit}) provides an upper bound on the regret $R_T$ as follows:
\begin{align}
	\label{eq:boundRT}
	R_T
	\lesssim
	\underset{\mathrm{stability~terms}}
		{\uuline{
		\sum_{t=1}^T  
		\eta_t z_t
		}
		}
		+
		\underset{\mathrm{penalty~terms}}
		{\uwave{
		\frac{1}{\eta_1}
		h_1
		+
		\sum_{t=2}^T 
		\left( \frac{1}{ \eta_t } - \frac{1}{\eta_{t-1}} \right)
		h_t
		}} ,
\end{align}
where $z_t$ and $h_t$ vary 
depending on the problem setup
and the regularizer function $\psi$.
For example,
in
the Hedge algorithm,
i.e.,
when \eqref{eq:defFTRLx} is specified by
$\cX = \cP(K) = \{ x \in [0,1]^K \mid \| x \|_1 = 1 \}$,
$\hat{f}_t(x) =  \ell_t^\top x$ with $\ell_t \in [0,1]^K$,
and $\psi(x)$ is the negative Shannon entropy,
$z_t$ and $h_t$ are bounded as
$z_t \le 
O\left( \sum_{i=1}^K \ell_{ti}^2 x_{ti} \right)
\le
O\left( \ell_t^\top x_t \right)
\le O(1)$
and
$h_t \le -\psi(x_t) \le \log K$.
In general FTRL,
a standard way of defining $h_t$ is
to set $h_t = \max_{x}\psi(x) - \psi(x_t)$.
Some concrete examples of $z_t$ will be discussed later,
such as in Section~\ref{sec:BOBW}.
Many existing methods for sequentially updating the learning rate adjust 
$\eta_t$ based solely on 
$z_t$ \citep{cesa2007improved,orabona2015scale,erven2011adaptive}.
Recently,
there has been consideration for adjusting the learning rate in response to $h_t$ as well \citep{ito2022nearly,tsuchiya2023best,kong2023best},
and approaches that adjust according to both $z_t$ and $h_t$ have emerged \citep{jin2023improved,tsuchiya2023stability}.
However, 
these update methods using $h_t$ are often somewhat ad-hoc,
designed for specific objectives (e.g., BOBW bounds),
and
the optimality of these update rules themselves have not been investigated.
More literature on FTRL with Tsallis entropy regularization is referenced in Appendix~\ref{sec:related}.

\subsection{Main contribution}
We first formulate the problem of choosing the learning rate
as an online decision-making problem to
minimize the right-hand side of
\eqref{eq:boundRT},
which is denoted by $F(\eta_{1:T}; z_{1:T}, h_{1:T})$.
For any update rule $\pi$,
we denote by $F^{\pi}(z_{1:T}, h_{1:T})$ the value of $F(\eta_{1:T}; z_{1:T}, h_{1:T})$ for $\eta_{1:T}$ determined by $\pi$,
where the update rule $\pi$ is specified as a series of functions:
$\pi = \{ \pi_t: (z_{1:t}, h_{1:t}) \mapsto \eta_t \}_{t \in \mathbb{N}}$.
We also define $\opt (z_{1:T}, h_{1:T})$
as
the minimum of $F(\eta_{1:T}; z_{1:T}, h_{1:T})$ achieved by 
the optimal sequence $\eta^*_1 \ge \eta^*_2 \ge \cdots \ge \eta^*_T$ of learning rates given
the entire series of $z_{1:T}$ and $h_{1:T}$ in advance.
Note that each $\eta^*_t$ may depend on $z_{1:T}$ and $h_{1:T}$ including the
``future feedback''
after the $t$-th
round.
To evaluate the performance of policies $\pi$ and the complexity of this online decision-making problem,
we focus on the \textit{competitive ratio} defined as
	$
	\CR(\pi; z_{1:T}, h_{1:T}) = \frac{F^{\pi}(z_{1:T}, h_{1:T})}{\opt (z_{1:T}, h_{1:T}) }
	$.

\begin{table}[t]
	\centering
	\caption{Upper bounds on $F^{\pi}$ achieved by proposed update rules $\pi$ for learning rates.}
	\label{tab:CR}
	\begin{tabular}{lcc}
		\toprule 
		Input for $\eta_t$ & $F^*$-dependent bound  & $(z_{1:T}, h_{1:T})$-dependent bound \\
		\midrule 
		$z_{1:t}, h_{1:t}$
		& $ 4 \sqrt{\xi} \opt $ 
		&
$\min\left\{ \sqrt{\sum_{t=1}^T z_t h_t \log T} ,  \sqrt{h_{\max} \sum_{t=1}^T z_t } \right\}$
		\\
		\addlinespace
		$z_{1:t-1}, h_{1:t-1}, \xi$
		& $ 4 \sqrt{\xi} \opt 
		+ O(z_{\max} + h_1)
		$ 
		& $ \min\left\{ \sqrt{\xi \sum_{t=1}^T z_t h_t \log T} ,  \sqrt{\xi h_{\max} \sum_{t=1}^T z_t } \right\}$
		\\
		\addlinespace
		$z_{1:t-1}, h_{1:t-1}, \hat{h}_{t}$
		& --
		& $\min\left\{ \sqrt{\sum_{t=1}^T z_t \hat{h}_{t+1} \log T} ,  \sqrt{\hat{h}_{\max} \sum_{t=1}^T z_t } \right\}$
		\\
		\midrule
		Lower bound
		&
		$\frac{\sqrt{T-1}}{\sqrt{T} + \xi}\sqrt{\xi} F^*$
		&
		--
		\\
		\bottomrule
	\end{tabular}
\end{table}

This study reveals that the optimal competitive ratio can be characterized by \textit{approximate monotonicity} of $h_{1:T}$.
For any fixed $\xi \ge 1$,
a sequence $h_{1:T}$ is called $\xi$-approximately monotone non-increasing if $\xi h_{t'} \ge h_t$ for all $t$ and $t' < t$.
Letting $H_{\xi}^T \subseteq \re_{>0}^T$ denote the set of all $\xi$-approximately monotone non-increasing sequences,
we have the following lower bound on the competitive ratio:
\begin{theorem}
	\label{thm:CRLB}
	For any $T \in \mathbb{N}$,
	any $\xi \ge 1$,
	and for any policy
	$\pi = \{ \pi_t: (z_{1:t}, h_{1:t}) \mapsto \eta_t \}$,
	there exist
	$z_{1:T} \in \re_{\ge 0}^T$
	and 
	$h_{1:T} \in H_{\xi}^T$ such that
	$\CR(\pi; z_{1:T}, h_{1:T}) \ge \frac{\sqrt{T-1}}{\sqrt{T}+\xi} \sqrt{\xi}$.
\end{theorem}
This lower bound implies that conditions on $h_{1:T}$ such as approximate monotonicity are essential
in order to establish non-trivial upper bounds on the competitive ratio.
The proof of this theorem is given in the appendix.

This paper also provides a policy $\pi = \{ \pi_t: (z_{1:t}, h_{1:t}) \mapsto \eta_t \}_{t \in \mathbb{N}}$
achieving a competitive-ratio upper bound that matches the lower bound in Theorem~\ref{thm:CRLB} up to a constant.
This policy is expressed by the solution of the following formula:
\begin{align}
	\label{eq:stpn}
	\eta_1 z_1
	=
	\frac{1}{\eta_1}h_1,
	\quad
	\eta_t z_t
	=
	\left(
		\frac{1}{\eta_t}
		-
		\frac{1}{\eta_{t-1}}
	\right)
	h_t
	\quad
	(t \ge 2),
\end{align}
i.e.,
the learning rate under which stability and penalty match in each round,
which is referred to as \textit{stability-penalty matching} (SPM) in this paper.
This formula of \eqref{eq:stpn} leads to the initialization of
$\eta_1 = \sqrt{z_1/h_1}$ and the update rule of
$\eta_{t} =
\frac{2}
{
	1
	+
	\sqrt{1 + 4 \eta_{t-1}^2 z_t / h_t}
}
\eta_{t-1}
$
for $t \ge 2$.
\begin{theorem}
	\label{thm:FUB}
	The policy $\pi = \{ \pi_t: (z_{1:t}, h_{1:t}) \mapsto \eta_t \}_{t \in \mathbb{N}}$ given by \eqref{eq:stpn} achieves
	$\CR(\pi; z_{1:T}, h_{1:T}) \le 4 \sqrt{\xi}$
	for any $\xi\ge 1$,
	$z_{1:T} \in \re_{\ge 0}^{T}$,
	and $h_{1:T} \in H_{\xi}^T$.
	In addition,
	this policy achieves 
	\begin{align}
		\label{eq:Fzh}
		F^{\pi}( z_{1:T}, h_{1:T})
		=
		O
		\left(
			\min \left\{ 
				\inf_{ \epsilon \ge \frac{1}{T}}
				\left\{
				\sqrt{
				\sum_{t=1}^T z_t h_t \log
				\left(
					\epsilon T
				\right)
				+
				\frac{
				z_{\max} h_{\max}
				}
				{\epsilon}
				}
				\right\},
				\sqrt{
					h_{\max}
					\sum_{t=1}^T z_t 
				}
			\right\}
		\right)
	\end{align}
	for any $z_{1:T} \in \re_{\ge 0}^{T}$ and $h_{1:T} \in \re_{>0}^T$,
	where 
	$h_{\max} = \max_{t \in [T]} h_{t}$
	and
	$z_{\max} = \max_{t \in [T]} z_{t}$.
\end{theorem}
The upper bound of \eqref{eq:Fzh} holds for any sequences of $z_{1:T}$ and $h_{1:T}$
without any requirement on the monotonicity.
Upper bounds in this form is useful in developing and analyzing BOBW bandit algorithms,
as can be seen in Section~\ref{sec:introBOBW} and Section~\ref{sec:BOBW}.

Theorems~\ref{thm:CRLB} and \ref{thm:FUB} together imply that
the tight competitive ratio under the condition on the approximate monotonicity of $h_{1:T}$
is of $\Theta(\sqrt{\xi})$,
and that such a tight competitive ratio is achieved by the policy given by \eqref{eq:stpn}.

We note that in the implementation of policy by \eqref{eq:stpn},
we need to know $h_t$ and $z_t$ at the time of determining $\eta_t$.
Such a knowledge is not always available in practice
as $h_z$ and $z_t$ may depend on $\eta_t$.
To deal with such situations, 
we also develop learning-rate policies that do not require values of $h_t$ and $z_t$
when determining $\eta_t$.
Bounds on $F^{\pi}$ achieved by such policies are summarized in Table~\ref{tab:CR}.
The ``Input'' row in this table represents the knowledge required in determining $\eta_t$.
For example,
if the input is $z_{1:t-1}, h_{1:t-1}, \hat{h}_t$,
the policy can be expressed as $\pi = \{ \pi_t : (z_{1:t-1}, h_{1:t-1}, \hat{h}_t) \mapsto \eta_t \}_{t\in \mathbb{N}}$.
The value $\hat{h}_t$ in this table is an arbitrary upper bound on $h_t$ that is available when determining $\eta_t$.
A typical example of $\hat{h}_t$ is to set $\hat{h}_t = h_{t-1}$,
which is justified when $h_t = O(h_{t-1})$ holds
and this condition can be ensured,
e.g.,
via Lemmas~\ref{lem:boundqr1} and \ref{lem:boundqr2} in this paper and via lemmas in \citet[Appendix C.3]{jin2023improved}.
Another example of $\hat{h}_t$ is to define
$\hat{h}_t = \xi \tilde{h}_{t-1} := \xi \min_{ s \in [t-1] } h_{s} $,
which is an upper bound of $h_t$ if $h_{1:T} \in H_{\xi}^T$.
Bounds shown in Table~\ref{tab:CR} are achieved by variants of the policy given by \eqref{eq:stpn},
which are provided in Section~\ref{sec:UB}.


\begin{table}[t]
	\centering
	\caption{Bounds on $z_t$ and regret for FTRL with $\alpha$-Tsallis entropy and SPM learning rates.
	Based on the values of $B(\alpha)$ in the upper table,
	we establish the BOBW regret bounds in the lower table.
	}
	\label{tab:SPAS}
	\begin{tabular}{lccccc}
		\toprule 
		Setting & Parameters & Bound on $z_t$ & $B(\alpha)$ & $\min_{\alpha} B(\alpha)$ \\
		\midrule 
		Multi-armed bandit 
		& $K$: \# arms
		& $\frac{1}{1-\alpha} \sum_{i\neq i^*} q_{ti}^{1-\alpha}$ 
		& $\frac{K-1}{\alpha (1-\alpha)}$  
		& $K-1$ \\
		\addlinespace
		Graph bandit 
		& $K$: \# arms,
		& $\frac{
		\zeta^{\alpha} (1-q_{ti^*})^{1-\alpha}
		}{1-\alpha} $
		& ${\frac{\zeta(K/\zeta)^{1-\alpha}}{\alpha (1-\alpha)} 
		}$  
		& 
		$
		\zeta \log_+ \left( \frac{K}{\zeta} \right)
		$
		\\
		&
		$\zeta$: independence number
		&
		&
		&
		\\
		\addlinespace
		Linear bandit
		& $K$: \# arms, 
		& $\frac{ d (1-q_{ti^*})^{1-\alpha}}{1-\alpha} $ 
		& ${\frac{ d K^{1-\alpha}}{\alpha (1-\alpha)} }$  
		& $d \log K$ \\
		&
		$d$: dimensionality
		&
		&
		&
		\\
		\addlinespace
		Contextual bandit
		& $M$: \# arms, $K$: \# experts
		& $\frac{K (1-q_{ti^*})^{1-\alpha}}{1-\alpha} $ 
		& ${\frac{ M K^{1-\alpha}}{\alpha (1-\alpha)} }$  
		& $M\log K$ \\
		\bottomrule
	\end{tabular}

	\begin{tabular}{lc}
		\addlinespace
		\toprule 
		Environment & Regret upper bound \\
		\midrule 
		Adversarial
		& 
		$
		O\left( 
			\sqrt{ B(\alpha) T } 
		\right)
		$
		\\
		\addlinespace
		Stochastic
		&
		$
		O \left(
			\frac{B(\alpha)}{\Delta_{\min}}\log_+ \left(
				\frac{\Delta_{\min}^2 T}{B(\alpha)}
			\right)
		\right)
		$
		\\
		\addlinespace
		Corrupted Stochastic
		&
		$
		O \left(
			\frac{B(\alpha)}{\Delta_{\min}}\log_+ \left(
				\frac{\Delta_{\min}^2 T}{B(\alpha)}
			\right)
			+
			\sqrt{
			\frac{C B(\alpha)}{\Delta_{\min}}\log_+ \left(
				\frac{\Delta_{\min}T}{C}
			\right)
			}
		\right)
		$
		\\
		\bottomrule
	\end{tabular}
\end{table}

\subsection{Application: best-of-both-worlds regret bounds}
\label{sec:introBOBW}
Bounds on $F$ dependent on $(z_{1:T}, h_{1:T} )$ such as \eqref{eq:Fzh}
are useful in developing BOBW bandit algorithms.
Examples dealt with in this paper are summarized in Table~\ref{tab:SPAS},
where we use the notation of $\log_+(x) = \max\{ 1, \log (x) \}$.
The regret bounds presented in Table~\ref{tab:SPAS} are achieved through an algorithmic framework detailed in Algorithm~\ref{alg:FTRL}.
%
Notably, Algorithm~\ref{alg:FTRL} in Section~\ref{sec:BOBW} adopts a methodology similar to those found in prior studies, such as \citet{auer2002nonstochastic,eldowa2023minimax,cesa2012combinatorial,zimmert2021tsallis}, with the distinct exceptions of its learning rate and regularization definitions.
Specifically, the employed regularization function utilizes a hybrid regularizer based on Tsallis entropy,
a concept previously explored in \citet{zimmert2019beating,tsuchiya2023best,masoudian2021improved,jin2023improved} and thus,
is not a novel contribution of this work.
The seminal contribution of this paper lies in the innovative update rules for the learning rate, demonstrating their effectiveness through
BOBW results.
These findings underscore the proposed SPM learning rates capability to significantly enhance performance.

As demonstrated in Table~\ref{tab:SPAS},
the SPM learning rates introduced in this paper achieve BOBW regret bounds with tight dependencies on $T$,
for any value of $\alpha \in (0, 1)$ in the $\alpha$-Tsallis entropy.
Specifically,
we attain an $O(\log T)$ bound in stochastic environments and an $O(\sqrt{T})$ bound in adversarial environments.
When designing FTRL-based BOBW algorithms,
various regularizers have been investigated,
including $\alpha$-Tsallis entropy~\citep{zimmert2021tsallis,jin2023improved},
log barrier~\citep{wei2018more,ito2022adversarially},
and Shannon entropy~\citep{ito2022nearly}.
However,
achieving bounds in a tight order for both stochastic and adversarial scenarios has been confined only when we use the $1/2$-Tsallis entropy regularizer (for instance, see \citealp{jin2023improved}).
This research marks the first instance of demonstrating optimality in terms of $T$ for $\alpha \neq 1/2$,
thereby presenting a method that allows the $\alpha$ parameter to be adjusted.
This adaptability ensures the achievement of optimal bounds relative to problem-specific parameters,
such as the independence number in graph bandits or the number of experts in contextual bandits.

The bounds presented in Table~\ref{tab:SPAS} of this study offer a comparative analysis with existing results as follows:
For Multi-Armed Bandits (MAB),
the primary benchmark is the work of \citet{jin2023improved}.
When $\alpha \neq 1/2$,
their bounds are $O(\sqrt{T \log T})$ in the adversarial setting and 
$O(\log T)$
in the stochastic setting.
Our study improves upon these by achieving $O(\sqrt{T})$ and $O(\log T)$, respectively,
thus presenting a tight dependency on $T$.
However, the bounds by \citet{jin2023improved} have advantages in considering the suboptimality gap of individual arms and allowing for multiple optimal arms.
In the case of $\alpha = 1/2$,
our results essentially replicate the bounds of Tsallis-INF \citep{zimmert2021tsallis,masoudian2021improved},
ignoring constant factors.
In graph bandits,
compared to the bounds by \citet{dann2023blackbox},
our results show an improved dependency on $\log K$,
achieving the same bounds as the algorithm by \citet{eldowa2023minimax} in the adversarial setting,
which are tight within a constant factor.
This can be seen as an extension of the adversarial-only results by \citet{eldowa2023minimax} to the BOBW results.
For contextual and linear bandits,
our bounds are nearly identical to those reported by \citet{dann2023blackbox},
but notably better when considering corrupted settings.
Our method achieves the refined bound of \citet{masoudian2021improved}, where 
$\log T $
is replaced by 
$\log (\Delta T / C)$,
indicating a superior performance of our bounds under certain conditions suggested by \citet{masoudian2021improved}.

The proposed approach,
similarly to other FTRL-based algorithms,
achieves bounds of $o(\sqrt{T})$-regret in stochastic regimes with adversarial corruption,
and more generally,
in adversarial regimes with self-bounding constraints~\citep{zimmert2021tsallis}.
The specific form of these bounds is presented in Table~\ref{tab:SPAS},
where $C \ge 0$ represents the corruption level,
indicating the magnitude of adversarial corruption.
Compared to the $O\left(\frac{1}{\Delta_{\min}} \log T + \sqrt{ \frac{C}{\Delta_{\min}} \log T } \right)$-bounds commonly found in existing studies~\citep{dann2023blackbox,zimmert2021tsallis},
our work refines these to a form of $
O\left(\frac{1}{\Delta_{\min}} \log_+ \left( \Delta_{\min}^2 T \right) + \sqrt{ \frac{C}{\Delta_{\min}} \log_+ \frac{ \Delta_{\min} T}{C} } \right)
$.
Similar bounds for the multi-armed bandit problem have been demonstrated by \citet{masoudian2021improved},
and for an understanding of the significance of these refined bounds,
we refer to this paper.
This study is the first to achieve such refined bounds for $\alpha$-Tsallis entropy with $\alpha \neq 1/2$ and to extend their applicability beyond multi-armed bandit problems.


\section{Problem Setup}
\label{sec:setup}
We consider the problem of updating the learning rate $\eta_t$ 
so that the
RHS
of \eqref{eq:boundRT} is minimized.
To this end,
we define
$F( \beta_{1:T} ; z_{1:T}, h_{1:T} )$ by
\begin{align}
  F( \beta_{1:T} ; z_{1:T}, h_{1:T} ) :=  \sum_{t=1}^T \left(   \frac{z_t}{\beta_t} + (\beta_{t}-\beta_{t-1}) h_t \right) ,
\end{align}
for $\beta_{1:T} = (\beta_t)_{t=1}^T \in \re_{>0}^T$,
$z_{1:T} = (z_t)_{t=1}^T \in \re_{\ge 0}^T$,
and
$h_{1:T} = (h_t)_{t=1}^T \in \re_{> 0}^T$,
where we let $\beta_0=0$ for notational simplicity.
The value of $F$ is equal to the main components of the
RHS
of \eqref{eq:boundRT},
under the variable transformation of $\beta_t = 1/\eta_t$.
We address a sequential decision-making problem where the objective is to choose $\beta_t$ based on the information
up to the $t$-th round,
given by 
$(z_{1:t}, h_{1:t})$,
or up to the $(t-1)$-th round,
given by 
$(z_{1:t-1}, h_{1:t-1})$,
with the goal of minimizing the value of $F$.

For any policy $\pi$ of choosing $\beta_t$,
let 
$F^{\pi}( z_{1:T}, h_{1:T} )$
be the value of
$F(\beta_{1:T}, z_{1:T}, h_{1:T} )$
for $\beta_{1:T}$ determined by $\pi$.
We measure the performance of policies $\pi$ based on the competitive ratio
given by
\begin{align}
	\label{eq:defCR}
	\CR (\pi; z_{1:T}, h_{1:T}) = \frac{F^\pi ( z_{1:T}, h_{1:T} )}{ \opt (z_{1:T}, h_{1:T}) },
\end{align}
where $\opt$ represents the minimum value of $F$ achieved by the offline optimization procedure
depending on the entire series of $z_{1:t}$ and $h_{1:t}$ in hindsight,
i.e.,
\begin{align}
	\opt (z_{1:T}, h_{1:T})
	=
	\inf \{
		F( \beta_{1:T} ; z_{1:T}, h_{1:T} ) \mid
		0 \le \beta_1 \le \beta_2 \le \cdots \le \beta_{T} 
	\}.
\end{align}%
\begin{remark}%
The constraint of $\beta_t \le \beta_{t+1}$ is equivalent to 
  the constraint that the learning rate $\eta_t$ is monotone non-increasing,
  i.e.,
  $\eta_t \ge \eta_{t+1}$.
  Although this constraint is not absolutely necessary in the algorithm design,
  it is often needed when obtaining regret upper bounds of the form of \eqref{eq:boundRT}.
\end{remark}
In interpreting the competitive ratio as defined in \eqref{eq:defCR} of this paper,
it is essential to be aware of its practical implications and limitations.
A smaller competitive ratio implies that,
upon fixing any sequences of $z_{1:T}$ and $h_{1:T}$,
the performance closely approximates that for the optimal sequence $\beta_{1:T}$ of learning rates.
However,
in the context of actual applications to FTRL,
the scenario is more complex because the values of $z_{t}$ and $h_{t}$ are influenced by the learning rate $\beta_{1:t}$ itself.
This leads to a critical insight:
Our competitive analysis does not incorporate how changes in the learning rate might affect 
$z_{1:T}$ and $h_{1:T}$ directly.
In other words, 
the ``optimality'' of the learning rate update rules, 
in the sense of the competitive ratio,
merely signifies optimality from the perspective of dependency on $z_{1:T}$ and $h_{1:T}$,
without considering the effects that learning rates have on $z_{1:T}$ and $h_{1:T}$.
Despite this limitation,
bounds dependent on $z_{1:T}$ and $h_{1:T}$ provide various benefits,
such as data-dependent bounds \citep{cesa2007improved,erven2011adaptive,de2014follow,orabona2015scale} and BOBW bounds \citep{zimmert2021tsallis} that are also discussed in Section~\ref{sec:BOBW},
and are thus of practical utility.

This paper shows that the optimal competitive ratio for some reasonable classes of policies can be characterized by
\textit{approximate monotonicity} of $h_{1:T}$:
\begin{definition}
	\label{def:approxmonotone}
	Let $\xi \ge 1$.
	We call a sequence $h_{1:T}$ is \textit{$\xi$-approximately non-increasing}
	if
	$ 
	\xi h_{t'} 
	\ge
	h_{t}
	$
	holds
	for any $t$ and $t'$ such that $t' < t$.
\end{definition}
Note that $1$-approximately non-increasing sequences are monotone non-increasing.
For any $\xi \ge 1$,
let $H_{\xi}^T$ denote the set of 
$\xi$-approximately non-increasing sequences,
i.e.,
\begin{align}
	\label{eq:defHxiT}
	H_{\xi}^T
	=
	\left\{ 
		h_{1:T} \in \re_{> 0}^{T} \mid
		t' < t \Longrightarrow \xi h_{t'} \ge h_t
	\right\}.
\end{align}
In our analysis,
we use the following property of $\xi$-approximately non-increasing sequences:
\begin{lemma}
	\label{lem:apxmonotonicity}
	Suppose $h_{1:T} \in H_{\xi}^T$.
	Then,
	$\tilde{h}_{1:T} \in H_{1}^T$ defined by $\tilde{h}_t = \min_{s \in \{ 1, 2, \ldots , t \}} h_s$ satisfies
	$\tilde{h}_t \le h_t \le \xi \tilde{h}_t$ for all $t$.
\end{lemma}
This lemma implies that the parameter $\xi \ge 1$ represents the ratio of how well the sequence $h_{1:T}$ can be approximated by 
a monotone non-increasing sequence.
All omitted proofs are given in the appendix.
Sequence $\tilde{h}_{1:T} \in H_{1}^T$ given in Lemma~\ref{lem:apxmonotonicity}
will be utilized in Section~\ref{sec:UB}.
For any nonnegative integer $n \in \mathbb{Z}_{\ge 0}$,
we denote $[n] = \{ 1, 2, \ldots, n \}$.
We also use the natation of
$z_{\max} = \sup_{t} z_t$
and
$h_{\max} = \sup_{t} h_t$.


\section{Stability-Penalty Matching}
\label{sec:UB}
Assume that at the time of choosing $\beta_t$,
we are given an access to $\hat{h}_{t}$,
an upper bound or an approximated value of $h_t$.
Consider the following two update rules:
\begin{align}
	\label{eq:stability-known}
	\mbox{Rule 1}
	\quad
	&
	\pi = \left\{ 
		\pi_t:
		(z_{1:t}, h_{1:t-1}, \hat{h}_t)
		\mapsto
		\beta_t
	\right\}
	&
	&
	\beta_0 = 0,
	\quad
	\beta_{t} = \beta_{t-1} + \frac{z_t}{\beta_t \hat{h}_t}
	\quad
	(t \ge 1),
	\\
	\label{eq:stability-agnostic}
	\mbox{Rule 2}
	\quad
	&
	\pi = \left\{ 
		\pi_t:
		(z_{1:t-1}, h_{1:t-1}, \hat{h}_t)
		\mapsto
		\beta_t
	\right\}
	&
	&
	\beta_1 > 0,
	\quad
	\beta_{t} = \beta_{t-1} + \frac{z_{t-1}}{\beta_{t-1} \hat{h}_t}
	\quad
	(t \ge 2).
\end{align}
We set learning rates by $\eta_t = 1/\beta_t$ with $\beta_t$ given by these rules.
We refer to these update rule as \textit{stability-penalty-matching} (SPM) learning rate,
as they are designed so that the $t$-th stability term 
$
 z_t / \beta_t$ 
(or the $(t-1)$-th stability term 
$
 z_{t-1}/\beta_{t-1}$)
matches the $t$-th penalty term 
$
(\beta_t - \beta_{t-1}) h_t
$.

The update rule of \eqref{eq:stability-known} can be viewed as a quadratic equation in $\beta_t$,
whose positive solution is
$
\beta_t = 
\frac{\beta_{t-1}}{2}
\left(
	1 + \sqrt{1 + {z_t}/({\beta_{t-1}^2 \hat{h}_t})}
\right)
$.
Specifically in our analysis,
we consider two typical settings of $\hat{h}_t$:
One is to set $\hat{h}_t = h_t$,
and
the other is to set
\begin{align}
	\label{eq:defhath}
	\hat{h}_1 =
	\xi h_1,
	\quad
	\hat{h}_t
	=
	\xi \tilde{h}_{t-1}
	=
	\xi
	\min_{s \in [t-1]} 
	h_s
	\quad
	(t \ge 2),
\end{align}
where the latter
ensures
$h_{t} \le \hat{h}_{t+1} \le \hat{h}_{t}$
and $  \hat{h}_{t+1} \le \xi h_{t}$,
which are used in our analysis.
These inequalities follow from Lemma~\ref{lem:apxmonotonicity}.

\begin{remark}
	The SPM learning rate can replicate several existing learning rate update rules under certain parameter settings.
	For example,
	if we set $h_t = \bar{h}$ for all $t$,
	\eqref{eq:stability-known} and \eqref{eq:stability-agnostic} lead to
	$\beta_t = \Theta\left( \sqrt{ \bar{h} \sum_{s=1}^t z_s } \right)$
	and
	$\beta_t = \Theta\left(\beta_1 + \sqrt{ \bar{h} \sum_{s=1}^{t-1} z_s } \right)$,
	respectively,
	which correspond to AdaFTRL-type learning rates~\citep{cesa2007improved,de2014follow,erven2011adaptive,orabona2015scale,ito2021parameter}.
	This approach is known to achieve regret bounds of $O\left( \sqrt{\bar{h} \sum_{t=1}^T z_t }\right)$.
	By considering another example,
	when $z_t = \Theta(\hat{h}_t)$,
	\eqref{eq:stability-known} leads to $\beta_t = \Theta(\sqrt{t})$.
	As a corresponding case,
	in Tsallis-INF using the $1/2$-Tsallis entropy \citep{zimmert2021tsallis},
	we can see that $z_t \approx h_t$,
	and it is known to be advantageous to use a learning rate of $\beta_t = \Theta(\sqrt{t})$.
	Further,
	when we set $z_t = p_t^{1-\alpha}$ and $\hat{h}_t = p_{t-1}^{\alpha}$ for some $p_t \in (0,1)$ and $\alpha \in (0, 1)$,
	\eqref{eq:stability-agnostic} leads to
	$\beta_t = \Theta \left( \beta_1 + \sqrt{ \sum_{s=1}^{t-1} p_s^{1-2\alpha} } \right)$,
	which replicates the learning rate designed by \citet{jin2023improved} for FTRL-based MAB algorithms with $(1-\alpha)$-Tsallis entropy regularizers.
\end{remark}

We show that SPM update rules achieve the following:
\begin{theorem}
	\label{thm:CRUB}
  Suppose $h_t \le \hat{h}_{t}$ holds for all $t$.
  If $\beta_t$ is given by \eqref{eq:stability-known},
  it holds that
  \begin{align}
	\label{eq:Fbound1}
\!\!\!\!	F(\beta_{1:T}; z_{1:T}, h_{1:T})
	=
	O\!
	\left(\!
		\min \left\{\!
			\inf_{\epsilon \ge \frac{1}{T}}\! \left\{\!
			\sqrt{
			\sum_{t=1}^T z_t \hat{h}_t \log (\epsilon T)
			+
			\frac{z_{\max}\hat{h}_{\max}}{\epsilon}
			},
			\sqrt{
				\hat{h}_{\max}
				\sum_{t=1}^T z_t 
			}
			\right\}\!
		\right\}\!
	\right)\!.\!
  \end{align}
  If $\beta_t$ is given by \eqref{eq:stability-agnostic},
  it holds that
  \begin{align}
\lefteqn{
\!\!\!\!\!     
	F(\beta_{1:T}; z_{1:T}, h_{1:T})
}\nonumber\\
&
\!\!\!\!\!
	=
	O\!
	\left(\!
		\min\! \left\{ \!
			\inf_{\epsilon \ge \frac{1}{T}}\! \left\{\!
			\sqrt{
			\sum_{t=1}^T z_t \hat{h}_{t+1} \log (\epsilon T)
			+
			\frac{z_{\max}\hat{h}_{\max}}{\epsilon}
			},
			\sqrt{
				\hat{h}_{\max}
				\sum_{t=1}^T z_t 
			}
			\right\}\!
		\right\}
      +
      \frac{z_{\max}}{\beta_1}
      +
      \beta_{1} \hat{h}_{1}
	\!\right)\!.
	\label{eq:Fbound2}
  \end{align}
  We also have the following bounds dependent on $\opt$:
  \begin{itemize}
	\item If $\pi = \{ \pi_t : ( z_{1:t} , h_{1:t} ) \mapsto \beta_t \}$ is given by \eqref{eq:stability-known} with $\hat{h}_t = h_t$,
	it holds for any $T$, $\xi \ge 1$, $z_{1:T} \in \re_{\ge 0}^T$, and $h_{1:T}\in H_{\xi}^T$
	that
	$F^{\pi} ( z_{1:T}, h_{1:T} ) \le  4 \sqrt{ \xi } \opt( z_{1:T}, h_{1:T} )$.
	\item If $\pi = \{ \pi_t : ( z_{1:t-1} , h_{1:t} ) \mapsto \beta_t \}$ is given by \eqref{eq:stability-agnostic} with $\hat{h}_t = h_t$,
	it holds for any $T$, $\xi \ge 1$, $z_{1:T} \in \re_{\ge 0}^T$, and $h_{1:T}\in H_{\xi}^T$
	that 
	$F^{\pi} ( z_{1:T}, h_{1:T} ) \le  4 \sqrt{ \xi } \opt( z_{1:T}, h_{1:T} ) + O\left( z_{\max}{\beta_1} + \beta_1 \hat{h}_1 \right)  $.
	\item For any fixed $\xi \ge 1$,
	if $\pi = \{ \pi_t : ( z_{1:t} , h_{1:t-1} ) \mapsto \beta_t \}$ is given by \eqref{eq:stability-known} with $\hat{h}_t$ defined as \eqref{eq:defhath},
	it holds for any $T$, $z_{1:T} \in \re_{\ge 0}^T$, and $h_{1:T}\in H_{\xi}^T$
	that $F^{\pi} ( z_{1:T}, h_{1:T} ) \le 4 \sqrt{ \xi } \opt( z_{1:T}, h_{1:T} ) + O\left(\sqrt{\xi h_{\max} z_{\max} }\right)$.
	\item For any fixed $\xi \ge 1$,
	if $\pi = \{ \pi_t : ( z_{1:t-1} , h_{1:t-1} ) \mapsto \beta_t \}$ is given by \eqref{eq:stability-agnostic} with $\hat{h}_t$ defined as \eqref{eq:defhath},
	it holds for any $T$, $z_{1:T} \in \re_{\ge 0}^T$, and $h_{1:T}\in H_{\xi}^T$
	that $F^{\pi} ( z_{1:T}, h_{1:T} ) \le 4 \sqrt{ \xi } \opt( z_{1:T}, h_{1:T} ) + O\left( z_{\max}{\beta_1} + \beta_1 \hat{h}_1 \right)  $.
  \end{itemize}
\end{theorem}
\begin{corollary}
	For the class of policy $\left\{\pi = \{ \pi_t : (z_{1:t}, h_{1:t})\mapsto \beta_t \} \right\}$
	and for any $\xi \ge 1$,
	the competitive ratio is bounded as follows:
	\begin{align}
		\inf_{\pi = \{ \pi_t: (z_{1:t}, h_{1:t}) \mapsto \beta_t \}}
		\sup_{T \in \mathbb{N}, z_{1:T} \in \re_{\ge 0}^T, h_{1:T} \in H_{\xi}^T}
			\CR(\pi; z_{1:T}, h_{1:T}  )
		\in
		\left[
			\sqrt{\xi},
			4 \sqrt{\xi}
		\right].
	\end{align}
\end{corollary}

For any $z_{1:T} \in \re_{\ge 0}^T$ and $h_{1:T} \in \re_{>0}^T$,
define $G(z_{1:T}, h_{1:T})$ by
\begin{align}
	\label{eq:defG}
	G( z_{1:T}, h_{1:T} )
	=
	\sum_{t=1}^T
	\left(\sum_{s=1}^t \frac{z_s}{h_{s}} \right)^{-1/2} z_t .
\end{align}
Using this function $G$,
we can provide upper bounds on $F$ as follows:
\begin{lemma}
	\label{lem:FG}
	Suppose $h_t \le \hat{h}_t $ holds for all $t$.
	If $\beta_t$ is given by \eqref{eq:stability-known},
	$F(\beta_{1:T}; z_{1:T}, h_{1:T}) \le 2 G(z_{1:T}, \hat{h}_{1:T})$.
	If $\beta_t$ is given by \eqref{eq:stability-agnostic},
	$F(\beta_{1:T}; z_{1:T}, h_{1:T}) \le 2 G(z_{1:T}, \hat{h}_{2:T+1}) + 7 \frac{z_{\max}}{\beta_1} + \beta_1 h_1$.
\end{lemma}
The value of $G$ can be bounded as follows:
\begin{lemma}
	\label{lem:boundG}
	Let $\theta_0 > \theta_1 > \theta_2 > \cdots > \theta_J > 0$ be an arbitrary positive and monotone decreasing sequence
	such that $\theta_0 \ge h_{\max}$.
	Denote
	$\mathcal{T}_{j} = \{ t \in [T] \mid \theta_{j-1} \ge h_t > \theta_j \}$ for $j \in [J]$
	and $\mathcal{T}_{J+1} = \{ t \in [T] \mid \theta_{J} \ge h_t \}$.
	We then have
	$
		G(z_{1:T}, h_{1:T})
		\le
		2
		\sum_{j=1}^{ J+1}
		\sqrt{
			\theta_{j-1}
			\sum_{t \in \mathcal{T}_j} z_t
		}.
	$
	Consequently,
	by choosing $\theta_j = h_{\max} 2^{-j}$,
	we obtain
	\begin{align}
		G(z_{1:T},h_{1:T})
		\le
		\min\left\{
			\inf_{J \in \mathbb{N}} \left\{
			\sqrt{
				8 
				J
				\sum_{t=1}^T h_t z_t
			}
			+
			2 \sqrt{ 2^{-J} T h_{\max} z_{\max}}
			\right\},
			2 \sqrt{h_{\max} \sum_{t=1}^T z_t}
		\right\}.
	\end{align}
\end{lemma}
On the other hand,
$F^*(z_{1:T}, h_{1:T})$ can be bounded from below as follows:
\begin{lemma}
	\label{lem:LBopt}
	Let $\theta_0 > \theta_1 > \theta_2 > \cdots > \theta_J > \theta_{J+1}=0$ be an arbitrary positive and monotone decreasing sequence.
	Denote
	$\mathcal{T}_{j} = \{ t \in [T] \mid \theta_{j-1} \ge h_t > \theta_j \}$ for $j \in [J]$.
	Suppose that $h_{1:T}$ is a monotone non-increasing sequence.
	We then have
		$
		\opt \left(z_{1:T}, {h}_{1:T} \right)
		\ge
		2
		\sum_{j=1}^{J}
		\sqrt{
			(\theta_{j} - \theta_{j+1}) \sum_{t \in \mathcal{T}_j} z_t
		}.
		$
\end{lemma}
To see the relation between $F^*$ and $G$,
define $H(z_{1:T}, h_{1:T})$ by
\begin{align*}
	H(z_{1:T}, h_{1:T}) = 
	\sum_{j=1}^\infty \sqrt{ \theta_{j-1} \sum_{t \in \mathcal{T}_j} z_t } ,
	\quad
	\mbox{where}
	\quad
\theta_{j} = h_{\max} 2^{-j}, \quad \mathcal{T}_j = \{ t \in [T] | \theta_{j-1} \ge h_t > \theta_{j} \}
\end{align*}
with $\theta_{j} = h_{\max} 2^{-j}$ and $\mathcal{T}_j = \{ t \in [T] | \theta_{j-1} \ge h_t > \theta_{j} \}$.
Lemma~\ref{lem:boundG} implies
$G(z_{1:T}, h_{1:T}) \le 2 H(z_{1:T}, h_{1:T})$
holds for any $z_{1:T}$ and $h_{1:T}$.
Further,
Lemma~\ref{lem:LBopt} means that
$\opt(z_{1:T}, h_{1:T}) \ge H(z_{1:T}, h_{1:T})$ holds if $h_{1:T}$ is monotone non-increasing.
\begin{remark}
	For the policy $\pi$ given by \eqref{eq:stability-known} with $\hat{h}_t = h_t$,
	if $h_{1:T}$ is monotone non-increasing,
	we can see that 
	each of $F^{\pi}(z_{1:T}, h_{1:T})$,
	$\opt(z_{1:T}, h_{1:T})$,
	$G(z_{1:T}, h_{1:T})$
	and
	$H(z_{1:T}, h_{1:T})$ 
	is in the constant factor of the others.
	In fact,
	we have
	$
	F^\pi(z_{1:T}, h_{1:T})
	\le
	2 G(z_{1:T}, h_{1:T})
	\le
	4 H (z_{1:T}, h_{1:T})
	\le 
	4 \opt ( z_{1:T}, h_{1:T} )
	\le
	4 F^\pi( z_{1:T}, h_{1:T})
	$.
\end{remark}
\begin{lemma}
  \label{lem:Gopt}
	If $h_{1:T}$ is $\xi$-approximately non-increasing for some $\xi \ge 1$
	we have
	\begin{align}
		G(z_{1:T}, h_{1:T}) 
		\le 2 \sqrt{\xi} \opt(z_{1:T}, h_{1:T}) .
	\end{align}
\end{lemma}
\begin{lemma}
 \label{lem:HH4}
	If $h_{1:T+1} \in H_1^{T+1}$,
	we then have
		$
		H(z_{1:T}, h_{1:T})
		\le
		H(z_{1:T}, h_{2:T+1})
		+
		4 \sqrt{h_{\max} z_{\max}}
		$.
\end{lemma}

By using the lemmas presented so far, we can prove Theorem~\ref{thm:FUB}:
\quad\\
\textbf{Proof sketch of Theorem~\ref{thm:FUB}}\quad
	Bounds on $F$ of
	\eqref{eq:Fbound1} and \eqref{eq:Fbound2}
	immediately follow from Lemmas~\ref{lem:FG} and \ref{lem:boundG}.
	In the following,
	we show bounds that depend on $\opt$.
	Suppose $h_{1:T} \in H_{\xi}^{T}$.
	Then,
	$\tilde{h}_t := \min_{s \in [t] } h_{s}$ satisfies
	$\tilde{h}_{t} \le h_t 
	\le \xi \tilde{h}_t 
	\le \xi \tilde{h}_{t-1} 
	$ and
	$\tilde{h}_{1:T} \in H_{1}^T$,
	i.e.,
	$\tilde{h}_{t} \ge \tilde{h}_{t+1} $.
	Hence,
	if $\beta_{1:T}$ is given by \eqref{eq:stability-known} with $\hat{h}_t = h_t$,
	we have
	\begin{align}
		F(\beta_{1:T}; z_{1:T}, h_{1:T})
		\le
		2 G(z_{1:T}, h_{1:T})
		\le
		4 \sqrt{\xi} \opt (z_{1:T}, h_{1:T}),
	\end{align}
	where the first and second inequalities follow from Lemmas~\ref{lem:FG} and \ref{lem:Gopt},
	respectively.
	If $\beta_{1:T}$ is given by \eqref{eq:stability-known} with \eqref{eq:defhath}
	then we have
	\begin{align*}
		F(\beta_{1:T}; z_{1:T}, h_{1:T})
		&
		\le
		2 G(z_{1:T}, \hat{h}_{1:T})
		&
		\mbox{(Lemma~\ref{lem:FG})}
		\\
		&
		=
		2 G(z_{1:T}, \xi \tilde{h}_{0:T-1})
		=
		2\sqrt{\xi} G(z_{1:T}, \tilde{h}_{0:T-1})
		&
		\mbox{(Definitions of $\tilde{h}_t$ and $G$ in \eqref{eq:defG})}
		\\
		&
		\le
		4 \sqrt{\xi} H(z_{1:T}, \tilde{h}_{0:T-1})
		&
		\mbox{(Lemma~\ref{lem:boundG})}
		\\
		&
		\le
		4 \sqrt{\xi} 
		\left(
			H(z_{1:T}, \tilde{h}_{1:T})
			+
			4 \sqrt{h_{\max} z_{\max}}
		\right)
		&
		\mbox{(Lemma~\ref{lem:HH4})}
		\\
		&
		\le
		4 \sqrt{\xi} 
		\left(
			\opt (z_{1:T}, \tilde{h}_{1:T})
			+
			4 \sqrt{h_{\max} z_{\max}}
		\right)
		&
		\mbox{(Lemma~\ref{lem:LBopt})}
		\\
		&
		\le
		4 \sqrt{\xi} 
		\left(
			\opt (z_{1:T}, {h}_{1:T})
			+
			4 \sqrt{h_{\max} z_{\max}}
		\right).
		&
		\mbox{(Definition of $\opt$ and $\tilde{h}_t \le h_t$)}
	\end{align*}
	Other bounds can be shown in a similar manner.
	For a complete proof, please refer to Appendix~\ref{sec:ProofThmFUB}.

\section{Application: best-of-both-worlds bandit algorithm}
\label{sec:BOBW}
This section provides examples of best-of-both-worlds
bandit algorithms based on the stability-penalty-matching learning rate.
In problem examples in this paper,
we consider the following procedure of online learning:
A player is given
the number of actions $K$,
and some information of the setup before the game starts.
In each round of $t \in \{ 1, 2, \ldots \}$,
the environment chooses a loss vector $\ell_t \in [-1, 1]^K$ while
the player chooses an action $I(t) \in [K]$,
and then incurs the loss of $\ell_{t,I(t)} \in [-1, 1]$.
The available feedback and the structure behind $\ell_t$
are different depending on the problem setup.
The performance of the player is measured by the regret defined as follows:
\begin{align}
  R_T(i^*)
  =
  \E \left[
    \sum_{t=1}^T
    \ell_{t,I(t)}
    -
    \sum_{t=1}^T
    \ell_{t,i^*}
  \right],
  \quad
  R_T
  =
  \max_{i^* \in [K]}
  R_T(i^*).
\end{align}

Let $p_t \in \cP(K) = \{ p \in [0, 1]^K \mid \| p \|_1 = 1 \}$ denote
the distribution from which an action $I(t)$ is chosen,
i.e.,
$\Pr[ I_t = i | \cH_{t-1}] =  p_{ti}$,
where $\cH_{t-1} = \{ (\ell_s, I(s)) \}_{s=1}^{t-1}$.
In an \textit{adversarial regime},
the loss $\ell_t$ can be chosen in an adversarial manner depending on $\cH_{t-1}$.
Special cases such as stochastic environments,
in which $\ell_t$ independently follows an identical unknown distribution,
can be captured in the following regime:
\begin{definition}[Adversarial regime with a self-bounding constraint \citep{zimmert2021tsallis}]\label{def:ASC}%
For $\Delta \in \re_{\ge 0}^K$,
  $C \ge 0$,
  and $T \in \mathbb{N}$,
  the environment is in an \textit{adversarial regime with a $(\Delta, C, T)$ self-boundig constraint} if the regret is bounded from below as
  $R_T \ge R'_T - C$,
  where we define 
  \begin{align}
	\label{eq:defASC}
    R'_T 
    =
    \E
    \left[
    \sum_{t=1}^T \Delta_{I(t)} 
    \right]
    =
    \E
    \left[
    \sum_{t=1}^T \sum_{i=1}^K \Delta_i p_{ti}
    \right] .
  \end{align}
\end{definition}
As discussed in \citet[Section 5]{zimmert2021tsallis},
this regime includes stochastic environments with adversarial corruption,
where each $\Delta_i \ge 0$ represents the suboptimality gap for action $i$,
and $C$ corresponds to the magnitude of corruption.
Following prior studies such as \citep{zimmert2021tsallis} and \citep{jin2023improved},
we assume that there is a unique optimal action $i^* \in [K]$,
and that $\Delta_i > 0$ holds for all $i \in [K] \setminus \{ i^* \}$.
Denote $\Delta_{\min} = \min_{ i \in [K] \setminus \{ i^* \} } \Delta_i$.

\subsection{Algorithmic framework for best-of-both-worlds}
This subsection provide an algorithmic framework for online learning problems based on FTRL, which has been considered in a variety of problems 
including multi-armed bandits~\citep{auer2002nonstochastic,zimmert2021tsallis},
combinatorial semi-bandits~\citep{zimmert2019beating},
graph bandits~\citep{alon2017nonstochastic,eldowa2023minimax},
linear bandits~\citep{cesa2012combinatorial},
and contextual bandits~\citep{auer2002nonstochastic}.

Our algorithmic framework computes the probability distribution $q_t \in \cP(K)$ given by
\begin{align}
  \label{eq:defFTRL}
    q_t \in \argmin_{ p \in \cP(K) }
    \left\{
        \sum_{s=1}^{t-1}
        \linner 
            \hat{\ell}_s,
            p
        \rinner
        +
        \beta_t
        \psi(p)
        +
        \betab
        \psib(p)
    \right\},
\end{align}
where $\hat{\ell}_t$ is an unbiased estimator of $\ell_t$.
Regularizers $\psi$ and $\psib$ are defined as follows:
\begin{align}
  \label{eq:defTsallis}
    \psi ( p )
    =
    -
    \frac{1}{\alpha}
    \sum_{i=1}^K
    (p_i^\alpha - p_i) ,
    \quad
    \psib ( p )
    =
    -
    \frac{1}{\alphab}
    \sum_{i=1}^K
    (p_i^{\alphab} - p_i) ,
    \quad
    \mbox{where}
    \quad
    \alpha \in (0,1),
    ~
    \alphab = 1-\alpha.
\end{align}
We refer $\psi$ (and $\psib$) as the \textit{$\alpha$-Tsallis entropy} (and the $\alphab$-Tsallis entropy) in this paper.
The additional regularizer $\betab \psib$ is introduced to ensure the condition of $h_t = O(h_{t-1})$ is satisfied.
Similar techniques,
referred to as \textit{hybrid regularizers},
have also been used in existing studies such as
\citet{masoudian2022best},
\citet{tsuchiya2023stability},
and
\citet{jin2023improved}.
We then choose an action $I(t) \in [K]$ from the distribution
$p_t \in \cP(K) $
defined by 
\begin{align}
  \label{eq:defpt}
	p_t
	=
	\left(1-{\gamma_t} \right) q_t
	+
	\gamma_t p_0,
\end{align}
where $\gamma_t \in [0,1/2]$ and
$p_0 \in \cP(K)$ is a distribution which we refer to as the \textit{exploration basis}.
By a standard analysis of FTRL (e.g., Exercise 28.12 in \citealp{lattimore2020bandit}),
we have
\begin{align}
	R_T
	&
	\le
	\E \left[
		\sum_{t=1}^T
		\left(
			2 \gamma_t
			+
			\linner 
			\hat{\ell}_t, q_{t} - q_{t+1}
			\rinner
			-
			\beta_t D(q_{t+1}, q_t)
      +
			(\beta_t - \beta_{t-1}) 
      h_t
			+
      \betab
      h'
		\right)
	\right],
\end{align}
where 
$D(p, q) = \psi(p) - \psi(q) - \linner \nabla \psi(q), p - q \rinner$ is the Bregman divergence associated with $\psi$,
and
we define
$h_t = -\psi(q_t)$
and
$h' = -\psib(q_1) \le \frac{1}{\alphab}K^{1-\alphab}$.
We note that
$h_t \le h_1 = h_{\max}$ holds for all $t$.

To obtain BOBW regret bounds,
we design
$p_0$,
$\hat{\ell}_t$,
$\alpha$,
$\beta_t$,
$\betab$,
and
$\gamma_t \in [0, 1/2]$,
so that
\begin{align}
  \label{eq:BOBWcondition}
  h_{t}
  =
  O(h_{t-1}),
  \quad
  \E \left[
    2
    \gamma_t
    +
    \linner 
    \hat{\ell}_t, q_{t} - q_{t+1}
    \rinner
    -
    \beta_t D(q_{t+1}, q_t)
    |
    \cH_{t-1}
  \right]
  =
  O \left(
  \frac{z_t}{\beta_t}
  \right)
\end{align}
hold 
for some $z_t \in [0, z_{\max}]$.
We then have
$R_T \le \E \left[ F( \beta_{1:T}; z_{1:T}, h_{1:T} ) \right] + \betab h'$.
By applying the SPM update rule \eqref{eq:stability-agnostic} with $\hat{h}_t = h_{t-1}$,
we obtain
\begin{align}
  R_T 
  =
  O\left(
    \E \left[
    \min \left\{
      \sqrt{
        h_{1}
        \sum_{t=1}^T
        z_t
      },
      \inf_{\epsilon \ge \frac{1}{T}} \left\{
      \sqrt{
        \sum_{t=1}^T
        h_t
        z_t
        \log (\epsilon T)
	+
	\frac{h_1 z_{\max}}{\epsilon}
      }
      \right\}
    \right\}
    \right]
    +
    \kappa
  \right)
\end{align}
as a direct consequence of Theorem \ref{thm:CRUB},
where we denote
$
\kappa =
\frac{
z_{\max}
}{\beta_1}
+
\beta_1 h_1
+
\betab h' 
$.
In an adversarial regime with a $(\Delta, C, T)$ self-bounding constraint,
if
\begin{align}
  \label{eq:htztomega}
  h_t z_t \le
  \omega ( \Delta ) \cdot
  \linner \Delta, q_t \rinner
\end{align}
holds for some $\omega(\Delta) > 0$,
we have
$
R_T = O \left(
  \sqrt{
    (R_T + C)
    \omega(\Delta)
    \log T
  }
  +
  \kappa
\right)
$,
which implies
$
R_T
=
O\left(
  \omega(\Delta) \log T
  +
  \sqrt{C \omega(\Delta) \log T}
  +
  \kappa
\right)
$.

\begin{algorithm}[t]
    \caption{FTRL with Tsallis-entropy regularizers and SPM learning rates}
    \label{alg:FTRL}
    \begin{algorithmic}[1]
    \REQUIRE $K \in \mathbb{N}, 0 \le \alpha < 1$, $\beta_1 > 0$, $\betab \ge 0$, $p_0 \in \cP(K)$.
    \FOR{$t=1,2,\ldots$}
    \STATE Compute $q_t \in \cP(K)$ given by \eqref{eq:defFTRL} with $\psi(p)$ and $\psib(p)$ defined in \eqref{eq:defTsallis}.
    \STATE Set $h_t = -\psi(q_t)$ and $z_t \ge 0$ based on $q_t$. Compute $\gamma_t$ based on $z_t$ and $\beta_t$. Set $p_t$ by \eqref{eq:defpt}.
    \STATE Choose $I(t)$ so that $\Pr[ I(t) = i ] = p_{ti}$ and get feedback from the environment.
    \STATE Compute $\hat{\ell}_t$ based on the feedback.
    \STATE Set $\beta_{t+1}$ by he update rule of \eqref{eq:stability-agnostic} with $\hat{h}_{t+1} = h_t$.
    \ENDFOR
    \end{algorithmic}
\end{algorithm}

The proposed algorithm is summarized in Algorithm~\ref{alg:FTRL}.
We note that the input of $p_0$ is not required if $\gamma_t = 0$ for all $t$.
Feedback information from the environment and the construction of $\hat{\ell}_t$ vary with each problem setting.
From the discussion in this section,
we can show that Algorithm~\ref{alg:FTRL} achieves BOBW regret bounds as follows:
\begin{proposition}
  \label{lem:BOBW}
  Suppose that \eqref{eq:BOBWcondition} holds
  and that some $z_{\max} > 0$ satisfies
  $z_t \le z_{\max}$ for all $t$ with probablity $1$.
  Then
  Algorithm~\ref{alg:FTRL} achieves
  $R_T =
  O \left( \E \left[ \sqrt{h_1 \sum_{t=1}^T z_t} + \kappa \right] \right)
  \le
  O \left(  \sqrt{h_1 z_{\max} T } + \kappa \right)
  $
  in adversarial regimes,
  where $\kappa = \frac{z_{\max}}{\beta_1} + \beta_1 h_1 + \betab \hb$.
  Further,
  in adversarial regimes with $(\Delta, C, T)$ self-bounding constraints,
  if \eqref{eq:htztomega} holds for some $\omega(\Delta)$,
  Algorithm~\ref{alg:FTRL} achieves
  \begin{align}
  R_T
  =
		O\left(
			\omega(\Delta) 
			\log_+ \left( 
				\frac{ h_1 z_{\max}  T  }{
					\omega(\Delta)^2
					+
					C \omega(\Delta)
				}
			\right)
			+
			\sqrt{
				C \omega(\Delta) 
				\log_+ \left( 
					\frac{ h_1 z_{\max} T  }{
						\omega(\Delta)^2
						+
						C \omega(\Delta)
					}
				\right)
			}
			+
			\kappa
		\right).
  \end{align}
\end{proposition}

In the subsections below,
we use the following notation:
\begin{align}
  \kappa
  =
  \frac{z_{\max}}{\beta_1} + \beta_1 h_1
  +
  \betab \hb,
  \quad
  q_{t*}
  =
  \min \left\{
\| q_t \|_{\infty},
    1- 
\| q_t \|_{\infty}
  \right\},
  \quad
  \tilde{q}_{ti}
  =
  \min \left\{
    q_{ti},
    q_{t*}
  \right\}.
\end{align}
In the following,
we demonstrate that using Algorithm~\ref{alg:FTRL},
we can achieve the BOBW regret bounds for multi-armed bandit and linear bandit problems as shown in Table~\ref{tab:SPAS}.
The results for graph bandits and for contextual bandits are described in Appendices~\ref{sec:graph} and \ref{sec:contextual},
respectively.

\subsection{Multi-armed bandit}
In the multi-armed bandit problem,
we assume that $\ell_t \in [0, 1]^{K}$
and that 
the player gets only feedback of the incurred loss of $\ell_{t,I(t)}$.
We set arbitrary $\alpha \in (0, 1)$
and set
\begin{align}
  \label{eq:paramMAB}
  \beta_1 \ge
    \frac{4K}{1-\alpha},
  \quad
  z_t = \frac{1}{1-\alpha} \sum_{i =1}^{K} \tilde{q}_{ti}^{1-\alpha},
  \quad
	\gamma_t = 0,
	\quad
	\hat{\ell}_{ti}
	=
	\frac{\mathbf{1}[I(t) = i]}{p_{ti}} \ell_{ti}.
\end{align}
In addition,
we set 
$\betab \ge 0$ as follows:
\begin{align}
  \label{eq:paramMAB2}
  \alpha \le \frac{1}{2}
  ~
  \Longrightarrow
  ~
  \betab = 0,
  \quad
  \quad
  \alpha > \frac{1}{2}
  ~
  \Longrightarrow
  ~
  \betab
  \ge
    \frac{32K}{(1-\alpha)^2 \beta_1} .
\end{align}
As shown in
Appendix~\ref{sec:ProofMAB},
conditions \eqref{eq:paramMAB} and \eqref{eq:paramMAB2} are 
sufficient conditions for \eqref{eq:BOBWcondition}.
Further,
we can show that
$h_t = - \psi(q_t)$
and
$z_t $ in \eqref{eq:paramMAB}
satisfy
$h_1 z_t \le \frac{K-1}{\alpha (1-\alpha)}$ and
\eqref{eq:htztomega} with
\begin{align}
  \label{eq:defomegaMAB}
  \omega(\Delta)
  =
  \frac{2}{\alpha (1-\alpha)} 
  \left(
    \sum_{i \neq i^*} \Delta_i^{-\frac{ \alpha }{1-\alpha}}
  \right)^{1-\alpha}
  \left(
    \sum_{i \neq i^*} \Delta_i^{-\frac{1- \alpha }{\alpha}}
  \right)^{\alpha}
  \le
  2
  \frac{K-1}{\alpha (1-\alpha) \Delta_{\min}} .
\end{align}
Hence,
from Proposition~\ref{lem:BOBW},
we have the following:
\begin{theorem}
  \label{thm:BOBWMAB}
  For the $K$-armed bandit problem,
  Algorithm~\ref{alg:FTRL} with \eqref{eq:paramMAB} and \eqref{eq:paramMAB2}
  achieves BOBW regret bounds in Proposition~\ref{lem:BOBW} with $h_1 z_{\max} = O\left( \frac{K-1}{\alpha (1-\alpha)} \right)$ and
  $\omega(\Delta)$ given by \eqref{eq:defomegaMAB}.
\end{theorem}
Note that
if $\alpha = 1/2$ then
$\omega(\Delta) = O \left( \sum_{i\neq i^*} 1/\Delta_i \right) $,
which recovers the regret bounds shown by \citet{zimmert2021tsallis,masoudian2021improved}.

\subsection{Linear bandit}
In the \textit{linear bandit} problems,
each arm $i \in [K]$ is associated with a
$d$-dimensional \textit{feature vector} $\phi_i \in \re^d$.
The environment in each round determines a \textit{loss vector} $\theta_t \in \re^d$,
for which the loss $\ell_{ti} \in [-1, 1]$ satisfies
$\E[ \ell_{ti} | \theta_t ] = \linner \theta_t, \phi_i \rinner$.
After choosing an arm $I(t)$,
the player observes only the incurred loss of $\ell_{t, I(t)}$.
Without loss of generality,
we assume that $d \le K$ and that
$\{ \phi_i \}_{i=1}^K $ spans $\re^d$.
For any distribution $p \in \cP(K)$,
denote
\begin{align}
  S( p )
  =
  \sum_{i=1}^K
  p_i \phi_i \phi_i^\top
  =
  \E_{I \sim p} \left[
    \phi_{I}\phi_{I}^\top
  \right] .
\end{align}
Then,
there exists a distribution $p \in \cP(K)$ such that
$\phi_i^{\top} S(p)^{-1} \phi_i \le d$ (see, e.g., \citealp[Theorem 21.1]{lattimore2020bandit}).
We choose
$p_0 \in \cP(K)$ so that
\begin{align}
  \label{eq:defp0linear}
  \phi_i^{\top} S(p_0)^{-1} \phi_i^\top \le c d
  \quad
  (i \in [K])
\end{align}
holds for some $c = O(1)$.
Let $\alpha \ge 1/2$ and set
\begin{align}
  \label{eq:paramLinear}
  \beta_1 \ge
  \frac{8 c d}{1-\alpha},
  ~
  \betab \ge 
  \frac{32 d}{(1-\alpha)^2 \beta_1},
  ~
  z_t = \frac{d q_{t*}^{1-\alpha}}{1-\alpha} ,
  ~
	\gamma_t = \frac{4 c z_t}{\beta_t},
  ~
	\hat{\ell}_{ti}
	=
  \ell_{t, I(t)}
  \phi_{I(t)}^\top
  S(p_t)^{-1} 
  \phi_i .
\end{align}
If $p_0$ satisfies \eqref{eq:defp0linear} and
if parameters are given by
\eqref{eq:paramLinear},
then
\eqref{eq:BOBWcondition} holds.
Further,
$h_t = - \psi(q_t)$
and
$z_t $ in \eqref{eq:paramLinear}
satisfy
$h_1 z_t \le \frac{d}{\alpha (1-\alpha)} K^{1-\alpha}$ and
\eqref{eq:htztomega} with
$\omega(\Delta)$ defined as
\begin{align}
  \label{eq:defomegaLinear}
  \omega(\Delta)
  =
  \frac{d}{\alpha (1-\alpha)} 
  \Delta_{\min}^{\alpha-1}
  \left(
    \sum_{i \neq i^*} \Delta_i^{-\frac{\alpha }{1-\alpha}}
  \right)^{1-\alpha}
  \le
  \frac{dK^{1-\alpha}}{\alpha (1-\alpha) \Delta_{\min}} 
  .
\end{align}
Hence,
Proposition~\ref{lem:BOBW} leads to the following regret bounds:
\begin{theorem}
  \label{thm:BOBWLinear}
  For linear bandit problems of $K$ arms associated with $d$-dimensional vectors,
  Algorithm~\ref{alg:FTRL} with $p_0$ satisfying \eqref{eq:defp0linear} and parameters given by \eqref{eq:paramLinear}
  achieves BOBW regret bounds in Proposition~\ref{lem:BOBW} with 
  $h_1 z_{\max} = O\left(
	\frac{d K^{1-\alpha}}{\alpha (1-\alpha)}
\right)$
  and
  $\omega(\Delta)$ given by \eqref{eq:defomegaLinear}.
\end{theorem}
Note that we obtain
$\frac{dK^{1-\alpha}}{\alpha (1-\alpha)} = O (d \log K)$
by setting
$\alpha = 1 - \frac{1}{4 \log K}$,
which recovers the regret upper bound by \citet[Corollary 12]{dann2023blackbox}.

\bibliography{reference.bib}
\bibliographystyle{plainnat}

\newpage
\appendix

\section{Additional Related Work}
\label{sec:related}
\paragraph{Online Learning using Tsallis entropy}
To the best of our knowledge, the use of Tsallis entropy in online learning is first considered by~\citet{audibert2009minimax,abernethy15fighting}, 
in which they showed that FTRL with Tsallis entropy can achieve an $O(\sqrt{kT})$ regret in multi-armed bandits.

After that Tsallis entropy has been employed in many online decision-making problems:
FTRL with Tsallis entropy of exponent $\alpha = 1 - 1/\log(k/s)$, was used to exploit the sparsity of losses in multi-armed bandits~\citep{kwon16gains},
and FTRL with $(1 - 1 / \log k)$-Tsallis entropy was used to obtain an improved regret bound in the strongly observable graph bandit problem~\citep{zimmert19connections}.

The most relevant studies to this paper are ones aimed at constructing BOBW algorithms using FTRL with Tsallis entropy.
\citet{zimmert2021tsallis} showed for the first time that FTRL with $1/2$-Tsallis entropy can achieve a nearly optimal logarithmic regret, 
whose regret bound in stochastic regimes with adversarial corruptions is later improved by~\citet{masoudian2021improved}.
FTRL with $1/2$-Tsallis entropy was also proven to be powerful in combinatorial semi-bandits~\citep{zimmert2019beating}, in the delayed feedback setting, where the loss of the selected action is observed after a delay~\citep{zimmert20delays,masoudian2022best}, 
in multi-armed bandits with switching costs, where the learner needs to pay a cost when changing their actions~\citep{rouyer2021algorithm,amir2022better},
dueling bandits~\citep{saha2022versatile}, and MDPs~\citep{jin2020simultaneously,jin2021best}.

In addition to these applications, it is known that in the decoupling setting, where different actions can be chosen for exploration and exploitation, FTRL with $2/3$-Tsallis entropy can achieve a constant regret bound~\citep{rouyer20tsallis}.
Interestingly, even in the setting of heavy-tailed multi-armed bandits, 
where the $n$-th moment of loss is bounded by $\sigma^n$ for some $\sigma > 0$,
FTRL with Tsallis entropy with exponent $1/n$ can achieve a logarithmic regret~\citep{huang22adaptive}.
Furthermore, for the weakly observable setting in graph bandits and for the globally observable setting in partial monitoring, whose minimax regrets are $\Theta(T^{2/3})$, 
FTRL with $1/2$-Tsallis entropy and the complement version of Tsallis entropy play key roles in achieving BOBW guarantees~\citep{ito2022nearly,tsuchiya2023best}.

\paragraph{Adaptive Learning Rate}
Using an adaptive learning rate is one of the most common ways to design algorithms with a desired adaptivity.
In the literature, it has been standard to determine the adaptive learning rate by relying on the stability component in~\eqref{eq:boundRT} observed so far.
Typical examples are AdaGrad~\citep{mcmahan10adaptive,duchi11adaptive} in online convex optimization and its closely related algorithms that can achieve the first-order bounds~\citep{allenberg2006hannan,abernethy2012interior,wei2018more}
and the second-order bounds~\citep{cesa2007improved,erven2011adaptive,de2014follow,gaillard2014second,orabona2015scale,ito2022nearly,olkhovskaya2023first}

In contrast, some very recent studies improve the adaptivity of algorithms by designing an adaptive learning rate depending on the penalty component in~\eqref{eq:boundRT}, instead of the stability term.
To our knowledge, \citet{ito2022nearly} is the first attempt for such a design, where 
the authors aimed at constructing BOBW algorithms.
A natural question that arises here is whether we can construct an adaptive learning rate that depends on both the stability and penalty terms.

The stability-penalty-adaptive (SPA) learning rate is the first adaptive learning rate that can achieve such simultaneous adaptivity~\citep{tsuchiya2023stability}.
With the SPA learning rate, they proved that~\eqref{eq:boundRT} is roughly bounded by $O(\sqrt{ \sum_{t=1}^T z_t h_{t+1} \log T})$.
A comparison of the SPA learning rate and the SPM learning rate is discussed in the following.

\paragraph{Comparison of SPM learning rate against SPA learning rate}
There are several issues in the existing adaptive learning rate that depend on the penalty term.
The biggest issue is that their regret upper bounds in the adversarial regime (or in the worst-case) have extra $O(\sqrt{ \log T})$ factors, which is due to the loose analysis or the ``ad-hoc'' learning rate designs.
Although the SPA learning rate is designed in a generic form so that it can be used for generic regularizers, the authors focus only on the Shannon entropy, not investigating the use of Tsallis entropy. 
As mentioned earlier, Tsallis entropy has been proven to be powerful in many BOBW algorithms,
and our adaptive learning rate framework could be used for a wide range of online decision-making problems besides those presented in the paper.

At a high level, this study provides a non-ad-hoc, theoretically grounded adaptive learning rate design principle 
by rethinking the design of adaptive learning rate from the standpoint of competitive analysis. 
Consequently, we succeeded in constructing nearly optimal BOBW algorithms, totally removing the suboptimality caused by the existing ad-hoc design of adaptive learning rates.


\section{Lower Bound on the Competitive Ratio}
This section provides a proof on Theorem~\ref{thm:CRLB},
which provide a lower bound on the competitive ratio.
Note here that
we use the notation $\beta_t = 1/\eta_t$ as an alternative to $\eta_t$,
as introduced in Section~\ref{sec:setup},
in our discussion below.
\quad\\
\textbf{Proof of Theorem~\ref{thm:CRLB}} \quad
	Consider two problem instances $(z_{1:T}, h_{1:T})$ and $(z'_{1:T}, h_{1:T})$ defined as follows:
	$z_1 = z'_1 = 1$,
	$h_1 = 1$,
	and
	$z_t = 0$,
	$z'_t = 1$,
	$h_t = \xi$
	for $t=2, \ldots, T$.
	We then have
	\begin{align}
		\opt( z_{1:T}, h_{1:T} )
		=
		2,
		\quad
		\opt( z'_{1:T}, h_{1:T} )
		\le
		\min_{\beta_1}
		\left\{
		\frac{T}{\beta_1}
		+
		\beta_1
		\right\}
		=
		2\sqrt{T}.
	\end{align}
	For a policy $\pi$,
	denote
	$\beta_1 = \pi_1(z_1, h_1)=\pi_1(z'_1, h_1) = \pi_1(1, 1)$.
	We then have
	\begin{align}
		\nonumber
		F^\pi( z_{1:T}, h_{1:T} )
		&
		\ge
		\frac{1}{\beta_1}
		+
		\beta_1,
		\\
		F^\pi( z'_{1:T}, h_{1:T} )
		&
		\ge
		\min_{\beta_T \ge \beta_1}
		\left\{
		\frac{1}{\beta_1}
		+
		\beta_1
		+
		\frac{T-1}{\beta_T}
		+
		(\beta_T - \beta_1) \xi
		\right\}
		\nonumber
		\\
		&
		\ge
		\frac{1}{\beta_1} + \beta_1
		+
		2 \sqrt{\xi (T-1)} - \xi \beta_1 
		\ge 
		2 \sqrt{\xi (T-1)} - \xi \beta_1 
		.
	\end{align}
	We hence have
	\begin{align}
		\nonumber
		&
		\inf_{\pi}
		\sup_{z_{1:T} \in \re_{\ge 0}^T, h_{1:T} \in H_{\xi}^T }
		\CR(\pi; z_{1:T}, h_{1:T})
		\ge
		\inf_{\pi}
		\max
		\left\{
			\frac{
				F^{\pi}(z_{1:T}, h_{1:T})
			}{
				\opt (z_{1:T}, h_{1:T})
			},
			\frac{
				F^{\pi}(z'_{1:T}, h_{1:T})
			}{
				\opt (z'_{1:T}, h_{1:T})
			}
		\right\}
		\\
		\nonumber
		&
		\ge
		\inf_{\beta > 0}
		\max\left\{
			\frac{1}{2}
			\left(
				\frac{1}{\beta} + \beta
			\right),
			\sqrt{\frac{\xi (T-1)}{T}}
			-
			\frac{ \xi \beta}{2\sqrt{T}}
		\right\}
		\\
		&
		\ge
		\inf_{\beta > 0}
		\max\left\{
			\frac{\beta}{2},
			\sqrt{\frac{\xi (T-1)}{T}}
			-
			\frac{\xi \beta}{2\sqrt{T}}
		\right\}
		=
		\frac{\sqrt{T}}{\sqrt{T} + \xi}
		\sqrt{\frac{\xi (T-1)}{T}}
		=
		\frac{\sqrt{T-1}}{\sqrt{T} + \xi} \sqrt{\xi}.
	\end{align}
$\qed$

\section{Omitted Proofs in Sections~\ref{sec:setup} and \ref{sec:UB}}

\subsection{Proof of Lemma~\ref{lem:apxmonotonicity}}
\begin{proof}
	Let $\tilde{h}_t = \min_{s \in [t]} h_s$.
	Then it is clear that
	$\tilde{h}_t \le h_t$ and $\tilde{h}_{t+1} \le \tilde{h}_{t}$.
	Further,
	it follows from the assumption of $h_{1:T} \in H_{\xi}^T$
	and the definition of $H_{\xi}^T$ in \eqref{eq:defHxiT}
	that
	\begin{align}
		\xi \tilde{h}_t
		=
		\min \left\{ \xi h_t, \min_{s < t} \{ \xi  h_{s} \} \right\}
		\ge
		\min \left\{ \xi h_t, h_t  \right\}
		=
		h_t ,
	\end{align}
	which completes the proof.
\end{proof}



\subsection{Proof of Lemma~\ref{lem:FG}}
\begin{proof}
	We first consider the case in which $\beta_t$ is given by \eqref{eq:stability-known}.
	We then have
	\begin{align}
		\label{eq:FG1}
		F( \beta_{1:T}; z_{1:T}, h_{1:T} )
		=
		\sum_{t=1}^T
		\left(
		\frac{z_t}{\beta_{t}}
		+
		(\beta_{t}-\beta_{t-1}) h_t
		\right)
		\le
		\sum_{t=1}^T
		\left(
		\frac{z_t}{\beta_{t}}
		+
		(\beta_{t}-\beta_{t-1}) \hat{h}_t
		\right)
		=
		2
		\sum_{t=1}^T
		\frac{z_t}{\beta_{t}}.
	\end{align}
	Further,
	it follows from \eqref{eq:stability-known} that
	\begin{align}
		\label{eq:FG2}
		\beta_t^2 = \beta_t \beta_{t-1} + \frac{z_t}{\hat{h}_t}
		\ge
		\beta_{t-1}^2 + 
		\frac{z_t}{\hat{h}_t}
		=
		\sum_{s=1}^t 
		\frac{z_s}{\hat{h}_s}.
	\end{align}
	By combining \eqref{eq:FG1} and \eqref{eq:FG2},
	we obtain $F(\beta_{1:T}; z_{1:T}, h_{1:T}) \le 2 G(z_{1:T}, \hat{h}_{1:T})$.

	We next consider the case of \eqref{eq:stability-agnostic}.
	We then have
	\begin{align}
		\nonumber
		F( \beta_{1:T}; z_{1:T}, h_{1:T} )
		&
		\le
		\frac{z_1}{\beta_{1}}
		+
		\beta_{1} {h}_1
		+
		\sum_{t=2}^T
		\left(
		\frac{z_t}{\beta_{t}}
		+
		(\beta_{t}-\beta_{t-1}) \hat{h}_t
		\right)
		\\
		&
		\nonumber
		=
		\frac{z_1}{\beta_{1}}
		+
		\beta_{1} {h}_1
		+
		\sum_{t=2}^T
		\left(
		\frac{z_t}{\beta_{t}}
		+
		\frac{z_{t-1}}{\beta_{t-1}}
		\right)
		\\
		&
		\le
		\beta_{1} {h}_1
		+
		2
		\sum_{t=1}^T
		\frac{z_t}{\beta_{t}}.
		\label{eq:FG3}
	\end{align}
	Further,
	for any $t\ge 2$,
	it follows from \eqref{eq:stability-known} that
	\begin{align}
		\label{eq:FG4}
		\beta_t^2 = \beta_t^2 + 2 \frac{z_{t-1}}{\hat{h}_t} + \left( \frac{z_{t-1}^2}{\beta_{t-1} \hat{h}_t} \right)^2
		\ge
		\beta_{t-1}^2 + 
		2
		\frac{z_{t-1}}{\hat{h}_t}
		=
		\beta_1^2
		+
		2
		\sum_{s=2}^t 
		\frac{z_{s-1}}{\hat{h}_s},
	\end{align}
	which implies that $\beta'_t := \sqrt{ \beta_1^2 +   2 \sum_{s=2}^t \frac{z_{s-1}}{\hat{h}_s}}$
	satisfies $\beta'_t \le \beta_t$.
	Denote
	$\mathcal{T} = \{ t \in [T] \mid \beta_{t+1}' \le \sqrt{2} \beta_t' \}$
	and $\mathcal{T}^c = [T] \setminus \mathcal{T} = \{t \in [T] \mid \beta_{t+1}' > \sqrt{2} \beta_{t}' \}$.
	We then have
	\begin{align*}
		\sum_{t=1}^{T}
		\frac{z_t}{\beta_t}
		&
		\le
		\sum_{t=1}^{T}
		\frac{z_t}{\beta_t'}
		=
		\sum_{t \in \mathcal{T}}
		\frac{z_t}{\beta_t'}
		+
		\sum_{t \in \mathcal{T}^c}
		\frac{z_t}{\beta_t'}
		\\
		&
		\le
		\sqrt{2}
		\sum_{t \in \mathcal{T}}
		\frac{z_t}{\beta_{t+1}'}
		+
		\sum_{t \in \mathcal{T}^c}
		\frac{z_{\max}}{\beta_t'}
		\\
		&
		\le
		\sqrt{2}
		\sum_{t \in \mathcal{T}}
		\frac{z_t}{\beta_{t+1}'}
		+
		\sum_{s=0}^{\infty}
		\left(\frac{1}{\sqrt{2}}\right)^s
		\frac{z_{\max}}{ \beta_1}
		\\
		&
		\le
		\sqrt{2}
		\sum_{t \in \mathcal{T}}
		\frac{z_t}{\beta_{t+1}'}
		+
		\frac{1}{1 - 1/\sqrt{2}}
		\frac{z_{\max}}{\beta_1}
		\le
		G(z_{1:T}, \hat{h}_{2:T+1})
		+
		(2+\sqrt{2})
		\frac{z_{\max}}{\beta_1}.
	\end{align*}
	Combining this with \eqref{eq:FG3},
	we obtain 
	$F(\beta_{1:T}; z_{1:T}, h_{1:T}) \le 2 G(z_{1:T}, \hat{h}_{2:T+1}) + 7 \frac{z_{\max}}{\beta_1} + \beta_1 h_1$.
\end{proof}

\subsection{Proof of Lemma~\ref{lem:boundG}}
\begin{proof}
	Inequality of the lemma can be shown as follows:
	\begin{align*}
		G(z_{1:T}, h_{1:T})
		&
		=
		\sum_{j=1}^{J+1}
		\sum_{t \in \mathcal{T}_j}
		\left(\sum_{s=1}^t \frac{z_s}{h_{s}} \right)^{-1/2} z_t 
		\le
		\sum_{j=1}^{J+1}
		\sum_{t \in \mathcal{T}_j}
		\left(\sum_{s \in \mathcal{T}_j \cap [t]} \frac{z_s}{h_{s}} \right)^{-1/2} z_t 
		\\
		&
		\le
		\sum_{j=1}^{J+1}
		\sum_{t \in \mathcal{T}_j}
		\left(\sum_{s \in \mathcal{T}_j \cap [t]} \frac{z_s}{\theta_{j-1}} \right)^{-1/2} z_t 
		=
		\sum_{j=1}^{J+1}
		\sqrt{\theta_{j-1}}
		\sum_{t \in \mathcal{T}_j}
		\frac{z_t}{\sqrt{ \sum_{s \in \mathcal{T}_j \cap [t] } z_s}}
		\\
		&
		\le
		2
		\sum_{j=1}^{J+1}
		\sqrt{\theta_{j-1}}
		\sum_{t \in \mathcal{T}_j}
		\frac{z_t}{
			\sqrt{ \sum_{s \in \mathcal{T}_j \cap [t] } z_s}
			+
			\sqrt{ \sum_{s \in \mathcal{T}_j \cap [t-1] } z_s}
		}
		\\
		&
		\le
		2
		\sum_{j=1}^{J+1}
		\sqrt{\theta_{j-1}}
		\sum_{t \in \mathcal{T}_j}
		\left(
			\sqrt{ \sum_{s \in \mathcal{T}_j \cap [t] } z_s}
			-
			\sqrt{ \sum_{s \in \mathcal{T}_j \cap [t-1] } z_s}
		\right)
		\le
		2
		\sum_{j=1}^{J+1}
		\sqrt{\theta_{j-1} \sum_{t \in \mathcal{T}_j} z_t}.
	\end{align*}
	By setting $J=0$ and $\theta_0 = h_{\max}$,
	we obtain
	\begin{align}
		G(z_{1:T}, h_{1:T})
		\le
		2 \sqrt{ h_{\max} \sum_{t=1}^T z_t}.
	\end{align}
	By setting 
	$\theta_j = 2^{-j} h_{\max}$ for $j = 0,1, \ldots, J$,
	we have
	\begin{align}
		\nonumber
		G(z_{1:T}, h_{1:T})
		&
		\le
		2
		\sum_{j=1}^{J+1}
		\sqrt{\theta_{j-1} \sum_{t \in \mathcal{T}_j} z_t}
		\le
		2
		\sum_{j=1}^{J}
		\sqrt{\frac{\theta_{j-1}}{\theta_{j}} \sum_{t \in \mathcal{T}_j} h_t z_t}
		+
		2\sqrt{{\theta_{J}} \sum_{t \in \mathcal{T}_J} z_t}
		\\
		\nonumber
		&
		=
		2
		\sum_{j=1}^{J}
		\sqrt{2 \sum_{t \in \mathcal{T}_j} h_t z_t}
		+
		2 \sqrt{ 2^{-J} h_{\max} \sum_{t \in \mathcal{T}_J} z_t}
		\\
		\nonumber
		&
		\le
		2 \sqrt{2 J \sum_{j=1}^J \sum_{t \in \mathcal{T}_j} h_t z_t}
		+
		2 \sqrt{ 2^{-J} h_{\max} 
		\sum_{t \in \mathcal{T}_J} z_t}
		\\
		&
		\le
		\sqrt{8 J \sum_{t = 1}^T h_t z_t}
		+
		2 \sqrt{ 2^{-J}h_{\max} z_{\max}T},
	\end{align}
	where the second inequality follows from $h_t > \theta_j$ for $j \in \mathcal{T}_j$
	and third inequality can be shown from the Cauchy-Schwarz inequality.
\end{proof}

\subsection{Proof of Lemma~\ref{lem:LBopt}}
\begin{proof}
	Define
	$\tau(j) = \max \{ t \in [T] \mid h_t > \theta_{j} \} $
	for $j=1,2,\ldots,J$
	and set $\tau(0) = 0$ and $\tau(J+1) = T$.
	We then have
	$
		\mathcal{T}_j = \{ \tau(j-1)+1, \ldots, \tau(j)  \}
	$
	for $j=1,2,\ldots, J+1$.
	For any non-decreasing sequence $\beta_{1:T} \in \re_{>0}^T$,
	we have
	\begin{align}
		\nonumber
		F \left( 
			\beta_{1:T}; z_{1:T}, h_{1:T}
		\right)
		&
		=
		\sum_{t=1}^{T}
		\left(
			\frac{z_t}{\beta_t}
			+
			(\beta_t - \beta_{t-1}) h_t
		\right)
		\ge
		\sum_{j=1}^{J}
		\sum_{t=\tau(j-1)+1}^{\tau(j)}
		\left(
			\frac{z_t}{\beta_t}
			+
			(\beta_t - \beta_{t-1}) h_t
		\right)
		\\
		\nonumber
		&
		\ge
		\sum_{j=1}^{J}
		\sum_{t=\tau(j-1)+1}^{\tau(j)}
		\left(
			\frac{z_t}{\beta_{\tau(j)}}
			+
			(\beta_t - \beta_{t-1}) \theta_{j}
		\right)
		\\
		\nonumber
		&
		=
		\sum_{j=1}^{J}
		\left(
			\frac{1}{\beta_{\tau(j)}}
			\sum_{t=\tau(j-1)+1}^{\tau(j)} z_t
			+
			\left(\beta_{\tau(j)} - \beta_{\tau(j-1)}\right) \theta_{j}
		\right)
		\\
		\nonumber
		&
		\ge
		\sum_{j=1}^{J}
		\left(
			\frac{1}{\beta_{\tau(j)}}
			\sum_{t=\tau(j-1)+1}^{\tau(j)} z_t
			+
			\beta_{\tau(j)} \left( \theta_{j} - \theta_{j+1} \right)
		\right)
		+ \beta_{\tau(J)} \theta_{J+1}
		- \beta_{\tau(0)} \theta_{1}
		\\
		&
		\ge
		2
		\sum_{j=1}^{J}
		\sqrt{
			\left( \theta_{j} - \theta_{j+1} \right)
			\sum_{t=\tau(j-1)+1}^{\tau(j)} z_t
		}
		=
		2
		\sum_{j=1}^{J}
		\sqrt{
			\left( \theta_{j} - \theta_{j+1} \right)
			\sum_{t \in \mathcal{T}_j} z_t
		},
	\end{align}
	where the last inequality follows from the AM-GM inequality and the fact that $\beta_{\tau(0)} = \beta_0 = 0$.
\end{proof}

\subsection{Proof of Lemma~\ref{lem:Gopt}}
\begin{proof}
	We first suppose that $h_{1:T}$ is monotone non-increasing.
	Then,
	from Lemma~\ref{lem:LBopt} with $\theta_j = h_{\max} 2^{-j}$,
	we have
	\begin{align}
		\nonumber
		\opt(z_{1:T},{h}_{1:T})
		&
		\ge
		2
		\sum_{j=1}^\infty \sqrt{ (\theta_{j} - \theta_{j+1}) \sum_{t \in \mathcal{T}_j} z_t } 
		=
		2
		\sum_{j=1}^\infty \sqrt{ (2^{-1} \theta_{j-1} - 2^{-2} \theta_{j-1}) \sum_{t \in \mathcal{T}_j} z_t } 
		\\
		&
		\ge
		\sum_{j=1}^\infty \sqrt{ \theta_{j-1} \sum_{t \in \mathcal{T}_j} z_t } 
		=
		H(z_{1:T}, {h}_{1:T}).
	\end{align}
	Further,
	as from Lemma~\ref{lem:boundG} implies
	$ G(z_{1:T}, h_{1:T}) \le 2 H(z_{1:T}, h_{1:T}) $,
	we have
	$G(z_{1:T}, h_{1:T}) \le 2 \opt(z_{1:T}, h_{1:T})$
	for non-increasing sequence $h_{1:T}$.

	We next consider the case in which $h_{1:T}$ is $\alpha$-approximately non-increasing.
	Define $\tilde{h}_t = \min_{s \in [t]}h_t$.
	Then $\tilde{h}_{1:T}$ is monotone non-increasing and
	it holds for any $t$ that
	$\tilde{h}_t \le h_{t} \le \alpha \tilde{h}_t$.
	We hence have
	\begin{align}
		G(z_{1:T}, h_{1:T})
		\le
		G(z_{1:T}, \alpha \tilde{h}_{1:T})
		\le
		2 \opt (z_{1:T}, \alpha \tilde{h}_{1:T})
		=
		2 \sqrt{\alpha} \opt (z_{1:T}, \tilde{h}_{1:T})
		\le
		2 \sqrt{\alpha} \opt (z_{1:T}, {h}_{1:T}),
	\end{align}
	which complete the proof.
\end{proof}

\subsection{Proof of Lemma~\ref{lem:HH4}}
\begin{proof}
	Denote $\tau(j) = \max\{t \in [T] \mid h_{t} > \theta_j \}$ and $\tau'(j) = \max\{t \in [T] \mid h_{t+1} > \theta_j \}$ for $j\ge 1$
	and $\tau(0) = \tau'(0) = 0$.
	We then have $\tau'(j) \le \tau(j) \le \tau'(j) + 1$.
	We hence have
	\begin{align}
		\nonumber
		H(z_{1:T}, h_{1:T})
		&
		=
		\sum_{j=1}^{\infty}
		\sqrt{
			\theta_{j-1}
			\sum_{t=\tau(j-1)+1}^{\tau(j)} z_t
		}
		\le
		\sum_{j=1}^{\infty}
		\sqrt{
			\theta_{j-1}
			\sum_{t=\tau'(j-1)+1}^{\tau'(j)+1} z_t
		}
		\\
		\nonumber
		&
		\le
		\sum_{j=1}^{\infty}
		\left(
		\sqrt{
			\theta_{j-1}
			\sum_{t=\tau'(j-1)+1}^{\tau'(j)} z_t
		}
		+
		\sqrt{\theta_j z_{\max}}
		\right)
		\\
		\nonumber
		&
		=
		H(z_{1:T}, h_{2:T+1})
		+
		\sqrt{h_{\max} z_{\max}}
		\sum_{j=1}^\infty \sqrt{2^{1-j}}
		\\
		&
		\le
		H(z_{1:T}, h_{2:T+1})
		+
		4
		\sqrt{h_{\max} z_{\max}}.
	\end{align}
\end{proof}

\subsection{Proof of Theorem~\ref{thm:FUB}}
\label{sec:ProofThmFUB}
\begin{proof}
	Bounds on $F$ of
	\eqref{eq:Fbound1} and \eqref{eq:Fbound2}
	immediately follow from Lemmas~\ref{lem:FG} and \ref{lem:boundG}.
	In the following,
	we show bounds that depend on $\opt$.
	Suppose $h_{1:T} \in H_{\xi}^{T}$.
	Then,
	$\tilde{h}_t := \min_{s \in [t] } h_{s}$ satisfies
	$\tilde{h}_{t} \le h_t 
	\le \xi \tilde{h}_t 
	\le \xi \tilde{h}_{t-1} 
	$ and
	$\tilde{h}_{1:T} \in H_{1}^T$,
	i.e.,
	$\tilde{h}_{t} \ge \tilde{h}_{t+1} $.
	Hence,
	if $\beta_{1:T}$ is given by \eqref{eq:stability-known} with $\hat{h}_t = h_t$,
	we have
	\begin{align}
		F(\beta_{1:T}; z_{1:T}, h_{1:T})
		\le
		2 G(z_{1:T}, h_{1:T})
		\le
		4 \sqrt{\xi} \opt (z_{1:T}, h_{1:T}),
	\end{align}
	where the first and second inequalities follow from Lemmas~\ref{lem:FG} and \ref{lem:Gopt},
	respectively.
	If $\beta_{1:T}$ is given by \eqref{eq:stability-known} with \eqref{eq:defhath}
	we then have
	\begin{align*}
		F(\beta_{1:T}; z_{1:T}, h_{1:T})
		&
		\le
		2 G(z_{1:T}, \hat{h}_{1:T})
		&
		\mbox{(Lemma~\ref{lem:FG})}
		\\
		&
		=
		2 G(z_{1:T}, \xi \tilde{h}_{0:T-1})
		=
		2\sqrt{\xi} G(z_{1:T}, \tilde{h}_{0:T-1})
		&
		\mbox{(Definitions of $\tilde{h}_t$ and $G$ \eqref{eq:defG})}
		\\
		&
		\le
		4 \sqrt{\xi} H(z_{1:T}, \tilde{h}_{0:T-1})
		&
		\mbox{(Lemma~\ref{lem:boundG})}
		\\
		&
		\le
		4 \sqrt{\xi} 
		\left(
			H(z_{1:T}, \tilde{h}_{1:T})
			+
			4 \sqrt{h_{\max} z_{\max}}
		\right)
		&
		\mbox{(Lemma~\ref{lem:HH4})}
		\\
		&
		\le
		4 \sqrt{\xi} 
		\left(
			\opt (z_{1:T}, \tilde{h}_{1:T})
			+
			4 \sqrt{h_{\max} z_{\max}}
		\right)
		&
		\mbox{(Lemma~\ref{lem:LBopt})}
		\\
		&
		\le
		4 \sqrt{\xi} 
		\left(
			\opt (z_{1:T}, {h}_{1:T})
			+
			4 \sqrt{h_{\max} z_{\max}}
		\right).
		&
		\mbox{(Definition of $\opt$ and $\tilde{h}_t \le h_t$)}
	\end{align*}
	We next consider the case in which $\beta_t$ is given by \eqref{eq:stability-agnostic}.
	Denote
	$
	\kappa = \frac{z_{\max}}{\beta_1} + \beta_{1}h_1
	$.
	If $\beta_t$ is given by \eqref{eq:stability-agnostic} with $\hat{h}_t = h_t$,
	we have
	\begin{align*}
		F(\beta_{1:T}; z_{1:T}, h_{1:T})
		&
		\le
		2 G(z_{1:T}, h_{2:T+1})
		+
		O \left(
			\kappa
		\right)
		&
		\mbox{(Lemma~\ref{lem:FG})}
		\\
		&
		\le
		2 G(z_{1:T}, \xi \tilde{h}_{1:T})
		+
		O \left(
			\kappa
		\right)
		&
		\mbox{($h_t \le \xi \tilde{h}_{t-1}$)}
		\\
		&
		=
		2 \sqrt{\xi} G(z_{1:T}, \tilde{h}_{1:T})
		+
		O \left(
			\kappa
		\right)
		&
		\mbox{(Definition~\eqref{eq:defG} of $G$)}
		\\
		&
		\le
		2 \sqrt{\xi} F^*(z_{1:T}, \tilde{h}_{1:T})
		+
		O \left(
			\kappa
		\right)
		&
		\mbox{(Lemmas~\ref{lem:boundG} and \ref{lem:LBopt})}
		\\
		&
		\le
		2 \sqrt{\xi} F^*(z_{1:T}, {h}_{1:T})
		+
		O \left(
			\kappa
		\right).
		&
		\mbox{(Definition of $\opt$ and $\tilde{h}_t \le h_t$)}
	\end{align*}
	If $\beta_{1:T}$ is given by \eqref{eq:stability-agnostic} with \eqref{eq:defhath},
	we have
	\begin{align*}
		F(\beta_{1:T}; z_{1:T}, h_{1:T})
		&
		\le
		2 G(z_{1:T}, \hat{h}_{2:T+1})
		+
		O \left(
			\kappa
		\right)
		&
		\mbox{(Lemma~\ref{lem:FG})}
		\\
		&
		\le
		2 G(z_{1:T}, \xi \tilde{h}_{1:T})
		+
		O \left(
			\kappa
		\right)
		&
		\mbox{(Definition of $\hat{h}_t$)}
		\\
		&
		=
		2 \sqrt{\xi} G(z_{1:T}, \tilde{h}_{1:T})
		+
		O \left(
			\kappa
		\right)
		&
		\mbox{(Definition~\eqref{eq:defG} of $G$)}
		\\
		&
		\le
		2 \sqrt{\xi} F^*(z_{1:T}, \tilde{h}_{1:T})
		+
		O \left(
			\kappa
		\right)
		&
		\mbox{(Lemmas~\ref{lem:boundG} and \ref{lem:LBopt})}
		\\
		&
		\le
		2 \sqrt{\xi} F^*(z_{1:T}, {h}_{1:T})
		+
		O \left(
			\kappa
		\right).
		&
		\mbox{(Definition of $\opt$ and $\tilde{h}_t \le h_t$)}
	\end{align*}
\end{proof}

\section{Analysis for Algorithm~\ref{alg:FTRL}: FTRL with SPM Learning Rates}
\subsection{Facts on FTRL}
\begin{lemma}
	\label{lem:FTRL}
	Suppose $q_t$ is given by \eqref{eq:defFTRL}.
	Then,
	it holds for any $p^* \in \cP(K)$ that
	\begin{align*}
		&
		\sum_{t=1}^T \linner \hat{\ell}_t, q_t - p^* \rinner
		\\
		&
		\le
		\sum_{t=1}^T
		\left(
			\linner \hat{\ell}_t,
			q_t - q_{t+1}
			\rinner
			-
			\beta_{t} 
			D( q_{t+1}, q_t )
			+
			(\beta_t - \beta_{t-1}) (\psi(p^*) - \psi(q_t))
		\right)
		+
		\psib(p^*)
		-
		\psib(q_1),
	\end{align*}
	where $D(p, q)$ denotes the Bregman divergence associated with $\psi$.
\end{lemma}
\begin{proof}
	We can apply a standard analytical technique,
	e.g.,
	in the proof of Lemma~1 by \citet{ito2021hybrid},
	as follows:
	\begin{align*}
		&
		\linner 
		\sum_{t=1}^T \hat{\ell}_t, p^*
		\rinner
		+
		\beta_{T} \psi(p^*)
		+
		\psib(p^*)
		\\
		&
		\ge
		\linner 
		\sum_{t=1}^T \hat{\ell}_t, q_{T+1}
		\rinner
		+
		\beta_{T} \psi(q_{T+1})
		+
		\psib(q_{T+1})
		\\
		&
		=
		\linner 
		\sum_{t=1}^{T-1} \hat{\ell}_t, q_{T+1}
		\rinner
		+
		\linner 
		\hat{\ell}_T, q_{T+1}
		\rinner
		+
		\beta_{T} \psi(q_{T+1})
		+
		\psib(q_{T+1})
		\\
		&
		\ge
		\linner 
		\sum_{t=1}^{T-1} \hat{\ell}_t, q_{T}
		\rinner
		+
		\linner 
		\hat{\ell}_T, q_{T+1}
		\rinner
		+
		\beta_{T} \psi(q_{T})
		+
		\psib(q_{T})
		+
		\beta_T D(q_{T+1}, q_T)
		\\
		&
		\ge
		\sum_{t=1}^T
		\left(
		\linner 
		\hat{\ell}_t, q_{t+1}
		\rinner
		+
		\beta_t D(q_{t+1}, q_t)
		+
		\left(
			\beta_{t-1} - \beta_t
		\right)
		\psi(q_{t})
		\right)
		+
		\psib(q_{1}),
	\end{align*}
	which implies that the desired inequality holds.
\end{proof}

\subsection{Facts on Tsallis entropy}
When $\psi$ is given by \eqref{eq:defTsallis},
then the Bregman divergence associated with $\psi$ is given by
\begin{align}
	D(p, q)
	=
	\frac{1}{\alpha}
	\sum_{i=1}^{K}
	\left(
	q_i^{\alpha} + \alpha (p_i-q_i) q_i^{\alpha - 1}
	- p_i^{\alpha}
	\right)
	=
	\sum_{i=1}^{K} d(p_i, q_i),
\end{align}
where we define
\begin{align}
	\label{eq:defd}
	d(p, q)
	:=
	\alpha^{-1}q^{\alpha} + (p-q) q^{\alpha - 1}
	- \alpha^{-1} p^{\alpha}
	\le
	\frac{1-\alpha}{2}
	\left(
		\min\{p, q\}
	\right)^{\alpha - 2}
	(p-q)^2.
\end{align}
\begin{lemma}[stability for one dimensional case]
	\label{lem:tsallis-onedim}
	Let $p, q \in (0, 1)$.
	Suppose $\ell \ge - \frac{1-\alpha}{2}q^{\alpha-1}$.
	We then have
	\begin{align}
		\label{eq:tsallis-onedim}
		\ell \cdot (q - p)
		-
		d(p, q)
		\le
		\frac{ 2 q^{2-\alpha} \ell^2}{1-\alpha}.
	\end{align}
\end{lemma}
\begin{proof}
	For any given $q$ and $\ell$,
	the left-hand side of \eqref{eq:tsallis-onedim} is concave in $p$.
	Hence,
	this is maximized when
	\begin{align}
		\label{eq:dp0}
		\frac{\mathrm{d}}{\mathrm{d}p}
		\left(
		\ell \cdot (q - p)
		-
		d(p, q)
		\right)
		=
		- \ell
		- 
		q^{\alpha - 1}
		+
		p^{\alpha - 1} = 0.
	\end{align}
	For such $p$,
	we have
	\begin{align}
		\label{eq:p2q}
		p 
		&
		= 
		(q^{\alpha-1} + \ell)^{\frac{1}{\alpha-1}}
		\le
		\left(q^{\alpha-1} - \frac{1-\alpha}{2}q^{\alpha-1}\right)^{\frac{1}{\alpha-1}}
		=
		q
		\left(1-\frac{1-\alpha}{2}\right)^{\frac{1}{\alpha - 1}}
		\\
		&
		=
		q
		\exp\left(
			\frac{1}{\alpha-1}
			\log
			\left(1+\frac{\alpha-1}{2}\right)
		\right)
		\le
		q
		\exp\left(
			\log 2
		\right)
		=
		2 q,
	\end{align}
	where the first equality follows from \eqref{eq:dp0} and the first inequality follows from the assumption of $\ell \ge - \frac{1-\alpha}{2}q^{\alpha-1}$.
	Further,
	from the intermediate value thoerem and the fact that $p^{\alpha - 2}$ is monotone decreasing in $p$,
	we have
	\begin{align*}
		|\ell|
		&
		=
		|p^{\alpha-1}-q^{\alpha-1}|
		\\
		&
		\ge
		\min\left\{
			|(\alpha-1)p^{\alpha-2}|,
			|(\alpha-1)q^{\alpha-2}|
		\right\}
		|p-q|
		\\
		&
		=
		(1-\alpha) \max\{ p, q \}^{\alpha - 2} |p-q|,
	\end{align*}
	where the first inequality follows from \eqref{eq:dp0} and
	the second inequality follows from the intermediate value thoerem.
	This implies
	\begin{align}
		\label{eq:p-qle}
		|p-q|
		\le
		\frac{1}{1-\alpha}
		\cdot \max\left\{ p, q \right\}^{2-\alpha}
		|\ell|.
	\end{align}
	As we have $\ell \cdot (q-p) = d(p, q) + d(q, p)$ for $p$ satisfying \eqref{eq:dp0},
	we have
	\begin{align}
		\ell \cdot (q - p)
		-
		d(p, q)
		&
		=
		d(q, p)
		\le
		\frac{1-\alpha}{2} ( \min\{ p, q \}^{\alpha - 2} ) ( p-q )^2
		\\
		&
		\le
		\frac{1}{2(1-\alpha)}  \min\{ p, q \}^{\alpha - 2}  (  \max \{ p, q \}^{2-\alpha}  \ell )^2,
	\end{align}
	where the first inequality follows from \eqref{eq:defd} and
	the second inequality follows from \eqref{eq:p-qle}.
	If $p \ge q$,
	as we have $p \leq 2q$ from \eqref{eq:p2q},
	it holds that
	\begin{align}
		\ell \cdot (q - p)
		-
		d(p, q)
		&
		\le
		\frac{1}{2(1-\alpha)}  \min\{ p, q \}^{\alpha - 2}  (  \max \{ p, q \}^{2-\alpha}  \ell )^2
		\\
		&
		\le
		\frac{1}{2(1-\alpha)}  q^{\alpha - 2}  ( (2q)^{2-\alpha}  \ell )^2
		=
		\frac{
		2 q^{2 - \alpha } \ell^2
		}{1-\alpha}  .
		\label{eq:lemonedim1}
	\end{align}
	If $p < q$,
	we have
	\begin{align}
		\ell \cdot (q - p)
		-
		d(p, q)
		&
		\le
		\ell \cdot (q - p)
		\le
		\frac{1}{1-\alpha}
		\cdot \max\left\{ p, q \right\}^{2-\alpha}
		\ell^2
		=
		\frac{
		q^{2-\alpha}
		\ell^2
		}{1-\alpha}
		\label{eq:lemonedim2}
	\end{align}
	where the first inequality follows from $d(p,q) \ge 0$,
	the second ineqluaity follows form \eqref{eq:p-qle},
	and the equality follows from the assumption of $p < q$.
	As \eqref{eq:lemonedim1} holds if $p\ge q$ and \eqref{eq:lemonedim2} holds otherwise,
	we have \eqref{eq:tsallis-onedim} for all $p$.
\end{proof}
\begin{lemma}[stabitlity for probability simplex]
	\label{lem:Tsallis-stab}
	Fix arbitrary $i^* \in [K]$
	and $q \in \cP(K)$.
	If $\ell_i \ge - \frac{1-\alpha}{4}q_i^{\alpha - 1}$ for all $i \in [K] $,
	we then have
	\begin{align}
		\label{eq:Tsallis-stab0}
		\linner \ell, q-p \rinner - D(p, q)
		\le
		\frac{4}{1-\alpha} \left(
			\sum_{i =1}^K q_i^{2-\alpha} \ell_i^2
		\right) 
	\end{align}
	for any $p \in \cP(K)$.
	If 
	$\ell_i \ge - \frac{1-\alpha}{4}q_i^{\alpha - 1}$ for all $i \in [K] \setminus \{ i^* \} $ and
	$\ell_{i^*} \le \frac{1-\alpha}{4}(1-q_{i^*})^{\alpha - 1}$,
	we then have
	\begin{align}
		\label{eq:Tsallis-stab1}
		\linner \ell, q-p \rinner - D(p, q)
		\le
		\frac{4}{1-\alpha} \left(
			\sum_{i\neq i^*} q_i^{2-\alpha} \ell_i^2
			+
			(1-q_{i^*})^{2-\alpha}
			\ell_{i^*}^{2}
		\right) 
	\end{align}
	for any $p \in \cP(K)$.
\end{lemma}
\begin{proof}
	From the definition of the Bregman divergence,
	we have
	\begin{align}
		\nonumber
		&
		\linner \ell, q-p \rinner - D(p, q)
		\\
		&
		\nonumber
		=
		\frac{1}{2}
		\sum_{i\neq i^*}
		\left(
			2 \ell_i \cdot (q_i - p_i)
			-
			d (p_i, q_i)
		\right)
		+
		\frac{1}{2}
		\left(
			2 \ell_{i^*} \cdot (q_{i^*} - p_{i^*})
			-
			d (p_{i^*}, q_{i^*})
			-
			\sum_{i \neq i^*} d(p_i, q_i)
		\right)
		\\
		&
		\nonumber
		\le
		\frac{1}{2}
		\sum_{i\neq i^*}
		\left(
			2 \ell_i \cdot (q_i - p_i)
			-
			d (p_i, q_i)
		\right)
		\\
		&
		\quad
		+
		\frac{1}{2}
		\min \left\{
			2 \ell_{i^*} \cdot (q_{i^*} - p_{i^*})
			-
			d (p_{i^*}, q_{i^*}),
			2 \ell_{i^*} \cdot (q_{i^*} - p_{i^*})
			-
			\sum_{i \neq i^*} d(p_i, q_i)
		\right\}.
		\label{eq:Tsallis-stab9}
	\end{align}
	From Lemma~\ref{lem:tsallis-onedim},
	if $\ell_{i} \ge - \frac{1-\alpha}{4} q_i^{\alpha - 1}$,
	we have
	\begin{align}
		2 \ell_{i} \cdot (q_{i} - p_{i})
		-
		d (p_{i}, q_{i})
		\le
		\frac{8 q_i^{2-\alpha} \ell_i^2}{1-\alpha}.
		\label{eq:Tsallis-stab8}
	\end{align}
	Hence,
	if it hold for all $i \in [K]$ that
	$\ell_{i} \ge - \frac{1-\alpha}{4} q_i^{\alpha - 1}$,
	we have \eqref{eq:Tsallis-stab0}.
	Further,
	we have
	\begin{align}
		q_{i^*}
		-
		p_{i^*}
		=
		(1 - p_{i^*})
		-
		(1 - q_{i^*})
		=
		\sum_{i \in [K] \setminus \{ i^* \}}
		(p_i - q_i).
	\end{align}
	As we have $(1-q_{i^*})^{\alpha - 1} \le q_i^{\alpha - 1}$ for any $i \in [K] \setminus \{i^*\}$,
	if $\ell_{i^*} \le \frac{1-\alpha}{4}(1-q_{i^*})^{\alpha - 1}$,
	we then have
	$- \ell_{i^*} \ge - \frac{1-\alpha}{4}q_{i}^{\alpha - 1}$
	for any $i \in [K] \setminus \{i^*\}$.
	Hence, Lemma~\ref{lem:tsallis-onedim} implies
	\begin{align}
		2 \ell_{i^*} \cdot (q_{i^*} - p_{i^*})
		-
		\sum_{i \in [K] \setminus \{ i^* \}} d(p_i, q_i)
		=
		\sum_{i \in [K] \setminus \{ i^* \}}
		\left(
		-
		2
		\ell_{i^*} \cdot (q_i - p_i)
		-
		d(p_i, q_i)
		\right)
		\\
		\le
		\frac{2}{1-\alpha}
		\sum_{i \in [K] \setminus \{ i^* \}}
		(2\ell_{i^*})^2 q_{i}^{2-\alpha}
		\le
		\frac{8}{1-\alpha}
		\ell_{i^*}^2 
		\left( \sum_{i \in [K] \setminus \{ i^* \}} q_i \right)^{2-\alpha}
		=
		\frac{8}{1-\alpha}
		( 1-q_{i^*})^{2-\alpha}
		\ell_{i^*}^2 .
	\end{align}
	By combining this with \eqref{eq:Tsallis-stab9} and \eqref{eq:Tsallis-stab8} for $i\in [K] \setminus \{ i^* \}$,
	we obtain \eqref{eq:Tsallis-stab1}.
\end{proof}
\begin{lemma}
	\label{lem:Tsallis-stab2}
	Fix arbitrary $q \in \cP(K)$ and let $i^* \in \argmax_{i \in [K]} q_i$.
	If $|\ell_i| \le \frac{1-\alpha}{4} \min \{ q_{i^*}, (1-q_{i^*}) \}^{\alpha - 1}$ holds for all $i \in [K]$,
	we have
	\begin{align}
		\linner \ell, q-p \rinner - D(p, q)
		\le
		\frac{4}{1-\alpha} \left(
			\sum_{i\neq i^*} q_i^{2-\alpha} \ell_i^2
			+
			\min \{ q_{i^*}, (1-q_{i^*}) \}^{2-\alpha}
			\ell_{i^*}^{2}
		\right) 
	\end{align}
	for any $p \in \cP(K)$.
\end{lemma}
\begin{proof}
	As we have
	$q_i \le q_{i^*}$ 
	and $q_i \le 1 - q_{i^*}$ holds 
	for any $i \in [K] \setminus \{ i^* \}$,
	we have
	$|\ell_i| \le 
	\frac{1-\alpha}{4} \min \{ q_{i^*}, 1-q_{i^*} \}^{\alpha - 1}
	\le 
	\frac{1-\alpha}{4} q_{i}^{\alpha - 1}
	$ for all $i \in [K] \setminus \{ i^* \}$.
	If $q_{i^*} \le 1 - q_{i^*}$,
	from \eqref{eq:Tsallis-stab0} in Lemma~\ref{lem:Tsallis-stab},
	we have
	\begin{align}
		\linner \ell, q-p \rinner - D(p, q)
		\le
		\frac{4}{1-\alpha} \left(
			\sum_{i=1}^K q_i^{2-\alpha} \ell_i^2
		\right) 
		\le
		\frac{4}{1-\alpha} \left(
			\sum_{i\neq i^*} q_i^{2-\alpha} \ell_i^2
			+
			\min \{ q_{i^*}, 1-q_{i^*} \}^{2-\alpha}
			\ell_{i^*}^{2}
		\right) .
	\end{align}
	If $q_{i^*} > 1 - q_{i^*}$,
	from \eqref{eq:Tsallis-stab1} in Lemma~\ref{lem:Tsallis-stab},
	we have
	\begin{align}
		\nonumber
		\linner \ell, q-p \rinner - D(p, q)
		&
		\le
		\frac{4}{1-\alpha} \left(
			\sum_{i\neq i^*} q_i^{2-\alpha} \ell_i^2
			+
			(1-q_{i^*})^{2-\alpha}
			\ell_{i^*}^{2}
		\right) 
		\\
		&
		\le
		\frac{4}{1-\alpha} \left(
			\sum_{i\neq i^*} q_i^{2-\alpha} \ell_i^2
			+
			\min \{ q_{i^*}, 1-q_{i^*} \}^{2-\alpha}
			\ell_{i^*}^{2}
		\right) .
	\end{align}
\end{proof}

\begin{lemma}
	\label{lem:psiqr}
	Fix arbitrary $\omega > 1$.
	For $q, r \in \cP(K)$,
	suppose that $r_i \le \omega q_i $ holds for all $i$.
	We then have
	$- \psi(r) \le - (1 + (\omega-1) \alpha) \psi(q) \le - \omega \psi(q)$.
\end{lemma}
\begin{proof}
	As $\psi(x)$ is a convex functions,
	we have
	\begin{align}
		\nonumber
		\psi(q) - \psi(r)
		&
		\le
		\linner
		\nabla \psi(q) , q - r
		\rinner
		=
		-
		\frac{1}{\alpha}
		\sum_{i=1}^K
		(\alpha q_i^{\alpha-1} - 1)(q_i - r_i)
		\\
		\nonumber
		&
		=
		-
		\sum_{i=1}^K
		(q_i^{\alpha-1} - 1)(q_i - r_i)
		=
		\sum_{i=1}^K
		(q_i^{\alpha-1} - 1)(r_i - q_i)
		\\
		&
		\le
		(\omega - 1)
		\sum_{i=1}^K
		(q_i^{\alpha-1} - 1)q_i
		=
		-
		(\omega - 1) \alpha \psi(q)
		\le
		-
		(\omega - 1)  \psi(q),
	\end{align}
	where the second inequality follows from the assumption of $r_i \le \omega q_i$.
	This implies that
	$- \psi(r) \le - (1 + (\omega-1) \alpha) \psi(q) \le - \omega \psi(q)$.
\end{proof}

\begin{lemma}
	\label{lem:boundqr1}
	Let $\omega = \sqrt{2}$.
	Suppose
	$q, r \in \cP(K)$ are given by
	\begin{align}
		q 
		&
		\in \argmin_{p \in \cP(K)} \left\{
			\linner L, p \rinner
			+
			\beta \psi(p)
			+
			\betab \psib(p)
		\right\},
		\\
		r
		&
		\in \argmin_{p \in \cP(K)} \left\{
			\linner L + \ell, p \rinner
			+
			\beta'
			\psi(p)
			+
			\betab
			\psib(p)
		\right\}
	\end{align}
	with
	\begin{align}
		\psi(p)
		=
		-
		\frac{1}{\alpha}
		\sum_{i=1}^{K} ( p_i^{\alpha} -  p_i ),
		\quad
		\psib(p)
		=
		-
		\frac{1}{\alphab}
		\sum_{i=1}^{K} ( p_i^{\alphab} -  p_i ),
	\end{align}
	where $0 \le \alphab < \alpha < 1$,
	$0 < \beta \le \beta'$,
	and
	$0 \le \betab $.
	Denote
	$q_* = \min\left\{ 1 - \max_{i \in [ K]}q_i,  \max_{i \in [ K]}q_i \right\}$.
	We also assume
	\begin{align}
		\label{eq:boundqr1condiell}
		\|
		\ell
		\|_{\infty}
		\le 
		\max \left\{
		\frac{1 - \omega^{\alpha-1}}{2} \beta q_*^{\alpha - 1},
		\frac{1 - \omega^{\alphab-1}}{2} \betab q_*^{\alphab - 1}
		\right\},
		\\
		0
		\le
		\beta' - \beta
		\le
		\max\left\{
			(1-\omega^{\alpha-1})
			\beta,
			\frac{ 1-\omega^{\alphab-1}}{\omega}
			\betab
			q_*^{\alphab - \alpha}
		\right\}.
		\label{eq:boundqr1condibeta}
	\end{align}
	We then have
		$
		r_i \le 2 q_i
		$
	for all $i \in [K]$.
\end{lemma}
\begin{proof}
	Let $i^* \in \argmax_{i \in [K]} q_i$.
	We then have $q_* = \min\{ q_{i^*}, 1 - q_{i^*} \}$.
	For any $i \in [K] \setminus \{ i^* \}$,
	we have
	$q_i \le q_{i^*}$ and
	$q_i = 1 - \sum_{i' \in [K] \setminus \{ i \}} q_{i'} \le 1 - q_{i^*}$,
	which implies $q_i \le q_{*}$.
	If $q_{i^*} > q_{*}$,
	we have $q_{i^*} > 1 - q_{i^*}$,
	which means $q_{i^*} > 1/2$.
	Hence,
	we can see that it
	suffices to show $r_i \le 2 q_i$ for all $i \in [K]$ such that $q_i \le q_*$.
	In fact,
	if $q_i > q_*$,
	such $i$ must be $i^*$ and $q_{i^*} > 1/2$,
	and therefore it is qlear that $ r_{i} \le 1 \le 2 q_{i} $.
	In the following,
	we focus on $i$ such that $q_i \le q_*$.

	We define a monotone decreasing function $g:\re_{>0} \rightarrow \re_{>0}$ by
	\begin{align}
		g(x) = \beta x^{\alpha-1} + \betab x^{\alphab-1}.
	\end{align}
	and define
	\begin{align}
		s
		\in \argmin_{p \in \cP(K)} \left\{
			\linner L + \ell, p \rinner
			+
			\beta \psi(p)
			+
			\betab \psib(p)
		\right\}.
	\end{align}

	We first show that $\omega^{-1} q_i \le s_i \le \omega q_i$ holds for all $i$ such that $q_i \le q_{*}$.
	From the first-order optimality condition,
	there exists $\lambda \in \re$ such that
	\begin{align}
		g(s_i)
		=
		g(q_i) + 
		\ell_i + \lambda
	\end{align}
	holds for all $i \in [K]$.
	If $\lambda < - \| \ell \|_{\infty}$,
	we have
	$g(s_i) < g(q_i)$ for all $i \in [K]$.
	Then,
	as $g$ is monotone decreasing,
	we have
	$s_i > q_i$ for all $i \in [K]$,
	which contradicts to $\| s \|_1 = \| q \|_1 = 1$.
	Hence,
	we have $\lambda \ge - \| \ell \|_{\infty}$.
	Similarly,
	we can see
	$\lambda \le \| \ell \|_{\infty}$.
	We hence have
	\begin{align}
		g(q_i) - 2 \| \ell \|_{\infty}
		\le
		g(s_i)
		\le
		g(q_i) + 2 \| \ell \|_{\infty}
	\end{align}
	for all $i \in [K]$.
	This implies that
	$\omega^{-1} q_i \le s_i \le \omega q_i$
	for all $i $ such that $q_i \le q_*$.
	In fact,
	we have
	\begin{align}
		\nonumber
		g(\omega q_i)
		&
		=
		\beta (\omega q_i)^{\alpha-1}
		+
		{\betab} (\omega q_i)^{{\alphab}-1}
		=
		\beta  q_i^{\alpha-1}
		+
		{\betab} q_i^{{\alphab}-1}
		-
		\beta (1 - \omega^{\alpha-1})q_i^{\alpha - 1}
		-
		{\betab} (1 - \omega^{\alpha - 1}) q_i^{{\alphab}-1}
		\\
		&
		\le
		g(q_i)
		-
		\beta (1 - \omega^{\alpha-1})q_*^{\alpha - 1}
		-
		{\betab} (1 - \omega^{\alpha - 1}) q_*^{{\alphab}-1}
		\le
		g(q_i) - 2 \| \ell \|_{\infty}
		\le 
		g(s_i),
		\\
		\nonumber
		g(\omega^{-1} q_i)
		&
		=
		\beta (\omega^{-1}q_i)^{\alpha-1}
		+
		\betab (\omega^{-1} q_i)^{\alphab - 1}
		=
		\beta q_i^{\alpha-1}
		+
		\betab q_i^{\alphab - 1}
		+
		\beta (\omega^{1-\alpha} - 1)q_i^{\alpha-1}
		+
		\betab (\omega^{1-\alpha} - 1)q_i^{\alphab - 1}
		\\
		&
		\ge
		g(q_i)
		+
		\beta (1 - \omega^{\alpha-1})q_*^{\alpha-1}
		+
		\betab (1 - \omega^{\alpha-1})q_*^{\alphab - 1}
		\ge
		g(q_i)
		+
		2 \| \ell \|_{\infty}
		\ge
		g(s_i) .
	\end{align}
	Since $g$ is a decreasing function,
	these implies that $\omega^{-1} q_i \le s_i \le \omega q_i$.

	We next show that $r_i \le \omega s_i$ holds for all $i$ such that $q_i \le q_*$.
	From the first-order optimality condition,
	there exists $\lambda \in \re$ such that
	\begin{align}
		g(r_i)
		+
		(\beta' - \beta) r_i^{\alpha-1}
		=
		g(s_i) + \lambda
	\end{align}
	holds for all $i \in [K]$.
	If $\lambda < 0$,
	we have
	$g(r_i) = g(s_i) + \lambda - (\beta' - \beta)r_i^{\alpha - 1} < g(s_i)$,
	which contradicts to $\|r \|_1 = \| s \|_1 = 1$.
	We hence have $\lambda \ge 0$,
	which implies
	\begin{align}
		g(r_i) + (\beta' - \beta) r_i^{\alpha-1} = g(s_i) + \lambda \ge g(s_i) .
	\end{align}
	For $i \in [K]$ such that $q_i \le q_*$,
	we have
	\begin{align}
		\nonumber
		&
		g(\omega s_i)
		+
		(\beta' - \beta)
		(\omega s_i)^{\alpha - 1}
		=
		\beta \omega^{\alpha-1} 
		s_i^{\alpha-1}
		+
		\betab \omega^{\alphab-1} 
		s_i^{\alphab-1}
		+
		(\beta' - \beta)
		s_i^{\alpha-1}
		\\
		\nonumber
		&
		=
		g(s_i)
		+
		\beta (\omega^{\alpha-1} -1)
		s_i^{\alpha-1}
		+
		\betab (\omega^{\alphab-1} -1)
		s_i^{\alphab-1}
		+
		(\beta' - \beta)
		s_i^{\alpha-1}
		\\
		\nonumber
		&
		\le
		g(s_i)
		+
		\beta (\omega^{\alpha-1} -1)
		s_i^{\alpha-1}
		+
		\betab (\omega^{\alphab-1} -1)
		s_i^{\alphab-1}
		+
		\left(
			(1-\omega^{\alpha-1})\beta
			+
			\frac{1-\omega^{\alphab - 1}}{\omega}
			\betab
			q_*^{\alphab - \alpha}
		\right)
		s_i^{\alpha-1}
		\\
		\nonumber
		&
		=
		g(s_i)
		+
		\betab (\omega^{\alphab-1} -1)
		s_i^{\alphab-1}
		+
		\left(\frac{q_*}{s_i}\right)^{\alphab - \alpha}
		\frac{1-\omega^{\alphab - 1}}{\omega}
		\betab
		s_i^{\alphab-1}
		\\
		\nonumber
		&
		=
		g(s_i)
		+
		\betab (\omega^{\alphab-1} -1)
		s_i^{\alphab-1}
		+
		(1-\omega^{\alphab - 1})
		\betab
		s_i^{\alphab-1}
		=
		g (s_i)
		\le
		g(r_i)
		+
		(\beta' - \beta)
		r_i^{\alpha - 1}.
	\end{align}
	This
	implies
	that
	$r_i \le \omega s_i $
	since the function of $x \mapsto g(x) + (\beta'-\beta)x^{\alpha - 1}$ is monotone decreasing.

	We hence have $r_i \le \omega s_i \le \omega^2 q_i = 2q_i$
	for all $i \in [K]$ such that $q_i \le q_*$,
	which completes the proof.
\end{proof}

\begin{lemma}
	\label{lem:boundqr2}
	Fix arbitrary $\omega \in (1, 2]$.
	Let $G = (V = [K], E)$ be an arbitrary undirected graph such that
	$(i,i) \in E$ holds for all $i \in V$,
	and let $N(i)$ denote the neighborhood of $i$,
	i.e.,
	$N(i) = \{ j  \in V \mid (i, j) \in E \}$.
	Suppose
	$q, r \in \cP(K)$ are given by
	\begin{align}
		q 
		&
		\in \argmin_{p \in \cP(K)} \left\{
			\linner L, p \rinner
			+
			\beta \psi(p)
			+
			\betab \psib(p)
		\right\},
		\\
		r
		&
		\in \argmin_{p \in \cP(K)} \left\{
			\linner L + \ell, p \rinner
			+
			\beta
			\psi(p)
			+
			\betab
			\psib(p)
		\right\}
	\end{align}
	with
	\begin{align}
		\psi(p)
		=
		-
		\frac{1}{\alpha}
		\sum_{i=1}^{K} ( p_i^{\alpha} -  p_i ),
		\quad
		\psib(p)
		=
		-
		\frac{1}{\alphab}
		\sum_{i=1}^{K} ( p_i^{\alphab} -  p_i ),
	\end{align}
	where $0 \le \alphab < \alpha < 1$,
	$\beta \ge \frac{K}{(\omega - 1)(1-\omega^{\alpha - 1})}  $,
	and
	$ \betab \ge 0 $.
	Suppose 
	$\ell$ is given by
	\begin{align}
		\label{eq:lemboundqr2condiell}
		\ell_i
		=
		\frac{\mathbf{1}[i' \in N(j)]}{\sum_{i' \in N(j)} q_{i'} } \ell'_i
	\end{align}
	for some $j$ and $\ell' \in [0,1]^K$.
	We then have
		$
		r_i \le \omega q_i
		$
	for all $i \in [K]$.
\end{lemma}
\begin{proof}
	Denote $Q_j = \sum_{i' \in N(j)} q_{i'}$.
	Define
	\begin{align}
		g(x) = \beta x^{\alpha-1} + \betab x^{\alphab-1} .
	\end{align}
	From the first-order optimality condition,
	there exists $\lambda \in \re$ such that
	\begin{align}
		g(r_i) = g(q_i) + \ell_i - \lambda
	\end{align}
	holds for all $i \in [K]$.
	As $g$ is monotone decreasing and $\| r \|_1 = \| q \|_1 = 1$,
	we have
	$0 \le \lambda \le \| \ell \|_{\infty}$
	We also have $ \| \ell \|_{\infty} \le Q_j$ fron the assumption of \eqref{eq:lemboundqr2condiell}.
	Suppose $Q_j \ge \epsilon$
	with $\epsilon := \frac{1}{\beta(1-\omega^{\alpha-1})} \le \frac{\omega-1}{K} \le \frac{1}{K}$.
	We then have
	$\lambda \le 1/Q_j \le 1/\epsilon$.
	For $i \in [K] $,
	we have
	\begin{align}
		g(r_i)
		\ge
		g(q_i)
		-
		\lambda
		\ge
		g(q_i)
		-
		1/\epsilon
		\ge
		g(\omega q_i)
		+
		\beta (1 - \omega^{\alpha - 1}) q_i^{\alpha - 1}
		-
		1/\epsilon
		\ge
		g( \omega q_i),
	\end{align}
	which implies $r_i \le \omega q_i$.
	Suppose $Q_j < \epsilon$.
	Then,
	noting that $\epsilon \le 1/K$,
	we can see that
	$i^* \in \argmax_{i \in [K]} q_i$ 
	is not included in $N(j)$ as
	$q_{i^*} \ge 1/K$.
	As we have $r_i \ge q_i$ for all $i \in [K] \setminus N(j)$,
	we have
	$r_{i^*} - q_{i^*} \le 
	\sum_{i \in [K] \setminus N(j)}( r_i - q_i ) = 
	\sum_{i \in N(j)}( q_i - r_i ) 
	\le Q_j < \epsilon$.
	Denote 
	$a := r_{i^*} / q_{i^*} \ge 1$.
	We then have
	$
		a
		=
		1
		+
		(r_{i^*} -q_{i^*})/q_{i^*}
		<
		1 + K \epsilon
    \le
    \omega
	$.
	In addition,
	we have
	\begin{align}
		g(q_i) - g(r_i)
		=
		\lambda - \ell_i
		\le
		\lambda
		=
		g(q_{i^*}) - g(r_{i^*})
		=
		g(q_{i^*}) - g(a q_{i^*})
		\le
		g(q_{i}) - g(a q_{i}),
	\end{align}
	where the last inequality follows from the fact that
	the function of $x \mapsto g(x) - g(ax)$ is monotone non-increasing for $a \ge 1$.
	This means that
	$g (a q_i) \le g(r_i)$,
	which implies
	$r_i \le a q_i $ as $g$ is monotone decreasing.
	By combining this with $a \le \omega$,
	we obtain 
	$r_i \le \omega q_i $ for all $i \in [K]$.
\end{proof}
\begin{lemma}
	\label{lem:boundqr3}
	Fix arbitrary $\omega > 1$.
	Suppose
	$q, r \in \cP(K)$ are given by
	\begin{align}
		q 
		&
		\in \argmin_{p \in \cP(K)} \left\{
			\linner L, p \rinner
			+
			\beta \psi(p)
			+
			\betab \psib(p)
		\right\},
		\\
		r
		&
		\in \argmin_{p \in \cP(K)} \left\{
			\linner L + \ell, p \rinner
			+
			\beta
			\psi(p)
			+
			\betab
			\psib(p)
		\right\}
	\end{align}
	with
	\begin{align}
		\psi(p)
		=
		-
		\frac{1}{\alpha}
		\sum_{i=1}^{K} ( p_i^{\alpha} -  p_i ),
		\quad
		\psib(p)
		=
		-
		\frac{1}{\alphab}
		\sum_{i=1}^{K} ( p_i^{\alphab} -  p_i ),
	\end{align}
	where $0 \le \alphab < \alpha < 1$,
	and
	$ \betab \ge 0 $.
	Suppose $\ell \in \re_{\ge 0}^K$ and
	\begin{align}
		\label{eq:sumqell}
		\sum_{i=1}^K q_i \ell_i \le 
		\frac{1}{K}
		\left(
		(1 - \omega^{\alpha-1})
		\beta 
		+
		(1 - \omega^{\alphab-1})
		\betab 
		\right)
	\end{align}
	We then have
		$
		r_i \le \omega q_i
		$
	for all $i \in [K]$.
\end{lemma}
\begin{proof}
	Define
	\begin{align}
		g(x) = \beta x^{\alpha-1} + \betab x^{\alphab-1} .
	\end{align}
	From the first-order optimality condition,
	there exists $\lambda \in \re$ such that
	\begin{align}
		\label{eq:gelllambda}
		g(r_i) = g(q_i) + \ell_i - \lambda
	\end{align}
	As $g$ is monotone decreasing and $\| r \|_1 = \| q \|_1 = 1$,
	we have
	$0 \le \lambda \le \| \ell \|_{\infty}$.
	Let 
	$g'(x) = (\alpha - 1) \beta x^{\alpha - 2} + (\alphab - 1) \betab x^{\alphab - 2} < 0$
	denote the derivative of $g(x)$.
	As $g$ is a convex function,
	we have
	\begin{align}
		\label{eq:gconv}
		g(r_i) \ge g(q_i) + g'(q_i) (r_i - q_i) .
	\end{align}
	Combining \eqref{eq:gelllambda} and \eqref{eq:gconv},
	we obtain
	\begin{align}
		\ell_i - \lambda \ge g'(q_i) (r_i - q_i),
	\end{align}
	which implies
	\begin{align}
		\sum_{i=1}^{K}
		(g'(q_i))^{-1} (\ell_i - \lambda ) \le 
		\sum_{i=1}^K ( r_i - q_i)
		=
		1 - 1
		=
		0.
	\end{align}
	We hence have
	\begin{align}
		\label{eq:lambdaell}
		\lambda
		\le
		\left(
		\sum_{i=1}^{K}
		\left(-g'(q_i)\right)^{-1}
		\right)^{-1}
		\sum_{i=1}^{K}
		\left( - g'(q_i) \right)^{-1} \ell_i .
	\end{align}
	Further,
	since it holds for any $x \in (0, 1)$ that
	\begin{align*}
		\left((1 - \alpha ) \beta + (1-\alphab) \betab \right) x^{-1}
		&
		\le
		(1 - \alpha ) \beta x^{\alpha - 2} + (1 - \alphab) \betab x^{\alphab - 2}
		=
		- g'(x)
		\\
		&
		\le
		\left((1 - \alpha ) \beta + (1-\alphab) \betab \right) x^{-2},
	\end{align*}
	we have
	\begin{align}
		\label{eq:gsum}
		\sum_{i=1}^{K}
		\left(-g'(q_i) \right)^{-1}
		\ge
		\sum_{i=1}^{K}
		\left((1 - \alpha ) \beta + (1-\alphab) \betab \right)^{-1} q_i^{2}
		\ge
		\left((1 - \alpha ) \beta + (1-\alphab) \betab \right)^{-1}
		\frac{1}{K}
	\end{align}
	and
	\begin{align}
		\label{eq:gellsum}
		\sum_{i=1}^{K}
		\left(-g'(q_i) \right)^{-1} \ell_i
		&
		\le
		\sum_{i=1}^{K}
		\left((1 - \alpha ) \beta + (1-\alphab) \betab \right)^{-1} q_i
		\ell_i
		\\
		&
		\le
		\left((1 - \alpha ) \beta + (1-\alphab) \betab \right)^{-1} 
		\frac{1}{K}
		(1 - \omega^{\alpha-1})
		\beta 
		+
		(1 - \omega^{\alphab-1})
		\betab ,
	\end{align}
	where the second inequality follows from the assumption of \eqref{eq:sumqell}.
	Combining \eqref{eq:lambdaell}, \eqref{eq:gsum} and \eqref{eq:gellsum},
	we obtain
	\begin{align}
		\label{eq:lambdabeta}
		\lambda
		\le
		(1 - \omega^{\alpha-1})
		\beta 
		+
		(1 - \omega^{\alphab-1})
		\betab .
	\end{align}
	Therefore,
	we have
	\begin{align}
		g(\omega q_i)
		&
		=
		\beta (\omega q_i)^{\alpha - 1}
		+
		\betab (\omega q_i)^{\alphab - 1}
		=
		g(q_i)
		-
		(1 - \omega^{\alpha-1})
		\beta q_i^{\alpha - 1}
		-
		(1 - \omega^{\alphab-1})
		\betab q_i^{\alphab - 1}
		\\
		&
		\le
		g(q_i)
		-
		(1 - \omega^{\alpha-1})
		\beta 
		-
		(1 - \omega^{\alphab-1})
		\betab 
		\le
		g(q_i)
		- \lambda
		\le
		g(r_i) 
	\end{align}
	for any $i \in [K]$,
	where the second and the last inequalities follow from \eqref{eq:lambdabeta} and \eqref{eq:gelllambda} with $\ell_i \ge 0$.
	Hence,
	as $g$ is monotone decreasing,
	we have $r_i \le \omega q_i$.
\end{proof}



\subsection{Proof of Proposition~\ref{lem:BOBW}}
\begin{proof}
	Fix arbitrary $i^* \in [K]$.
	Let $p^* \in \{ 0, 1 \}^K$ denote the indicator vector of $i^*$,
	i.e.,
	$p^*_{i^*} = 1$ and $p^*_i = 0$ for all $i \in [K] \setminus \{ i^* \}$.
	From the definition \eqref{eq:defpt} of $p_t$ and the assumption that $\hat{\ell}_t$ is an unbiased estimator of $\ell_t$,
	we have
	\begin{align}
		\nonumber
		R_T(i^*)
		&
		=
		\E \left[
			\sum_{t=1}^T \ell_{t, I(t)} - 
			\sum_{t=1}^T \ell_{t, i^*}
		\right]
		\\
		\nonumber
		&
		=
		\E \left[
			\sum_{t=1}^T \linner \ell_{t}, p_t - p^*  \rinner
		\right]
		\\
		\nonumber
		&
		=
		\E \left[
			\sum_{t=1}^T \linner \ell_{t}, q_t - p^*  \rinner
			+
			\sum_{t=1}^T \gamma_t \linner \ell_{t}, p_0 - q_t   \rinner
		\right]
		\\
		&
		\le
		\E \left[
			\sum_{t=1}^T \linner \hat{\ell}_{t}, q_t - p^*  \rinner
			+
			2
			\sum_{t=1}^T \gamma_t 
		\right].
	\end{align}
	From Lemma~\ref{lem:FTRL},
	we have
	\begin{align}
		\sum_{t=1} \linner \hat{\ell}_{t}, q_t - p^*  \rinner
		\le
		\sum_{t=1}^T
		\left(
			\linner 
			\hat{\ell}_t, q_{t} - q_{t+1}
			\rinner
			-
			\beta_t D(q_{t+1}, q_t)
			+
			(\beta_t - \beta_{t-1}) 
			h_t
			+
			\betab
			h'
		\right),
	\end{align}
	where 
	$D(p, q)$ represents the Bregman divergence associated with $\psi(p)$,
	i.e.,
	$D(p, q) = \psi(p) - \psi(q) - \linner \nabla \psi(q), p - q \rinner$,
	and
	we denote
	$h_t = -\psi(q_t)$
	and
	$\hb = -\psib(q_1) \le \frac{1}{\alphab}K^{1-\alphab}$.
	By combining these inequalities,
	we obtain
	\begin{align}
		\nonumber
		R_T(i^*)
		&
		\le
		\E \left[
		\sum_{t=1}^T
		\left(
			2\gamma_t
			+
			\linner 
			\hat{\ell}_t, q_{t} - q_{t+1}
			\rinner
			-
			\beta_t D(q_{t+1}, q_t)
			+
			(\beta_t - \beta_{t-1}) 
			h_t
			+
			\betab
			\hb
		\right)
		\right]
		\\
		\nonumber
		&
		=
		O
		\left(
		\E \left[
		\sum_{t=1}^T
		\left(
			\frac{z_t}{\beta_t}
			+
			(\beta_t - \beta_{t-1}) 
			h_{t-1}
			+
			\betab
			\hb
		\right)
		\right]
		\right)
		\\
		&
		=
		O \left(
			\E\left[
				F(\beta_{1:T}; z_{1:T}, h_{0:T-1} )
			\right]
			+
			\betab
			\hb
		\right),
		\label{eq:propRT}
	\end{align}
	where the second inequality follows from the assumption of \eqref{eq:BOBWcondition}
	and we define $h_{0} = h_1$.
	From Theorem~\ref{thm:FUB},
	if $\beta_t$ is given by \eqref{eq:stability-agnostic} with $\hat{h}_t = h_{t-1}$
	(which is clearly an upper bound on $h_{t-1}$),
	we have
	\begin{align}
		F(\beta_{1:T}; z_{1:T}, h_{0:T-1})
		&
		=
		O
		\left(
			\sqrt{
				{h}_1
				\sum_{t=1}^T z_t 
			}
			+
			\frac{z_{\max}}{\beta_1}
			+
			\beta_{1} \hat{h}_{1}
		\right),
		\label{eq:propF1}
		\\
		F(\beta_{1:T}; z_{1:T}, h_{0:T-1})
		&
		=
		O
		\left(
			\inf_{\epsilon \ge \frac{1}{T}} \left\{
			\sqrt{
			\sum_{t=1}^T z_t {h}_{t} \log (\epsilon T)
			+
			\frac{z_{\max}{h}_{1}}{\epsilon}
			}
			\right\}
			+
			\frac{z_{\max}}{\beta_1}
			+
			\beta_{1} \hat{h}_{1}
		\right).
		\label{eq:propF2}
	\end{align}
	By combining \eqref{eq:propRT} and \eqref{eq:propF1},
	we obtain
	$R_T =
	O \left( \E \left[ \sqrt{h_1 \sum_{t=1}^T z_t} + \kappa \right] \right)
	\le
	O \left(  \sqrt{h_1 z_{\max} T } + \kappa \right)
	$
	in adversarial regimes.

	We next consider the case of adversarial regimes with self-bounding constraints.
	By combining \eqref{eq:propRT}, \eqref{eq:propF2},
	and Jensen's inequality,
	we obtain
	\begin{align}
		R_T
		=
		O \left(
			\sqrt{
				\E \left[
			\sum_{t=1}^T z_t h_t
			\right]
			\log (\epsilon T)
			+
			\frac{z_{\max} h_1 }{\epsilon}
			}
			+
			\kappa
		\right)
	\end{align}
	for any $\epsilon \ge 1/T$.
	Under the condition of adversarial regimes with $(\Delta, C, T)$ self-bounding constraints,
	we have
	\begin{align}
		\E \left[ \sum_{t=1}^T z_t h_t \right]
		\le
		\omega(\Delta)
		\E \left[ \sum_{t=1}^T \linner \Delta, q_t \rinner \right]
		\le
		2 
		\omega(\Delta)
		\E \left[ \sum_{t=1}^T \linner \Delta, p_t \rinner \right]
		\le
		2 
		\omega(\Delta)
		(R_T + 2 C),
	\end{align}
	where the first inequality follows from \eqref{eq:htztomega},
	the second inequality follows from $p_{ti} = (1-\gamma_t)q_{ti} + \gamma_t p_{0i} \ge \frac{1}{2} q_{ti}$,
	and the last inequality follows from the assumption of self-bounding constraints given in Definition~\ref{def:ASC}.
	We hence have
	\begin{align}
		R_T
		=
		O \left(
			\sqrt{
				\omega(\Delta) (R_T + C)
			\log (\epsilon T)
			+
			\frac{z_{\max} h_1 }{\epsilon}
			}
			+
			\kappa
		\right),
	\end{align}
	which implies
	\begin{align}
		R_T
		=
		O\left(
			\omega(\Delta) \log (\epsilon T)
			+
			\sqrt{
				C \omega(\Delta) \log (\epsilon T)
				+
				\frac{z_{\max} h_1 }{\epsilon}
			}
			+
			\kappa
		\right).
	\end{align}
	We here used the fact that
	$X = O(\sqrt{AX}+B)$
	implies $X = O(A + B)$ for any $X, A, B \ge 0$.
	By setting
	\begin{align}
		\epsilon
		=
		\frac{ {z_{\max} h_1} }{
			\omega(\Delta)^2
			+
			C \omega(\Delta)
		},
	\end{align}
	we obtain
	\begin{align*}
		R_T
		=
		O\left(
			\omega(\Delta) 
			\log_+ \left( 
				\frac{ z_{\max} h_1 T  }{
					\omega(\Delta)^2
					+
					C \omega(\Delta)
				}
			\right)
			+
			\sqrt{
				C \omega(\Delta) 
				\log_+ \left( 
					\frac{ z_{\max} h_1 T  }{
						\omega(\Delta)^2
						+
						C \omega(\Delta)
					}
				\right)
			}
			+
			\kappa
		\right).
	\end{align*}
\end{proof}

\subsection{Multi-Armed Bandit: Proof of Theorem~\ref{thm:BOBWMAB}}
\label{sec:ProofMAB}
From Proposition~\ref{lem:BOBW},
it suffices to verify that conditions \eqref{eq:BOBWcondition} and \eqref{eq:htztomega} hold.
\paragraph{Verifying condition \eqref{eq:BOBWcondition}}
In the following,
we denote 
\begin{align}
	\tilde{I}(t) \in \argmax_{i \in [K]} q_{ti}.
\end{align}
We then have $p_{t,\tilde{I}(t)} = q_{t, \tilde{I}(t)} \ge 1/K$,
and hence
$\hat{\ell}_{t, \tilde{I}(t)} \le \frac{\ell_{t,\tilde{I}}}{p_{t, \tilde{I}(t)}} \le K \le \frac{(1-\alpha) \beta_1}{4} \le \frac{(1-\alpha) \beta_t}{4}$.
Hence,
from Lemma~\ref{lem:Tsallis-stab} with $i^* = \tilde{I}(t)$,
we have
\begin{align}
		\nonumber
  \E \left[
    \linner 
    \hat{\ell}_t,
    q_{t} - q_{t+1}
    \rinner
    -
    \beta_{t} D(q_{t+1}, q_t)
    |
    \cH_{t-1}
  \right]
  &
  \le
    \frac{4}{(1-\alpha)\beta_t}
    \left(
    \sum_{i \in [K] \setminus \{ \tilde{I}(t) \}}
    q_{ti}^{1-\alpha}
    +
    q_{t*}^{1-\alpha}
    \right)
  \\
  &
  \le
    \frac{8}{(1-\alpha)\beta_t}
    \sum_{i \in [K] \setminus \{ \tilde{I}(t) \}}
    q_{ti}^{1-\alpha}
    =
    O \left(
	\frac{z_t}{\beta_t}
    \right),
\end{align}
which implies that the second part of \eqref{eq:BOBWcondition} holds.

We next show that the first part of \eqref{eq:BOBWcondition} holds.
Define
$q'_t \in \argmin_{ p \in \cP(K)} \left\{ \linner \sum_{s=1}^{t-1} \hat{\ell}_s, p \rinner + \beta_{t+1} \psi( p ) + \betab \psib (p)  \right\}$.
We
show 
$q'_{ti} \le 2 q_{ti}$ and $ q_{t+1, i} \le 2 q'_{ti} $
by using 
Lemmas~\ref{lem:boundqr1} and \ref{lem:boundqr2},
respectively.
The condition for Lemma~\ref{lem:boundqr1} can be verified as follows:
From the definition of $z_t$,
we have
\begin{align}
  z_t \le \frac{K}{1-\alpha} q_{t*}^{1-\alpha}
\end{align}
and
\begin{align}
	\label{eq:LBht}
  h_t
  =
  -\psi(q_t)
  \ge
  \frac{q_{t*}^{\alpha}}{\alpha} (1 - 2^{\alpha - 1})
  \ge
  \frac{(1-\alpha)q_{t*}^{\alpha}}{4 \alpha}.
\end{align}
We hence have
\begin{align}
  \beta_{t+1} - \beta_{t}
  =
  \frac{z_t}{\beta_t \hat{h}_{t+1} }
  =
  \frac{z_t}{\beta_t {h}_{t} }
  \le
  \frac{4 \alpha K  q_{t*}^{1-2\alpha}}{\beta_1 (1-\alpha)^2 }.
\end{align}
Therefore,
from the definition of $\beta_1$ in \eqref{eq:paramMAB},
if $\alpha \le 1/2$,
we have
\begin{align}
  \frac{4 \alpha K  q_{t*}^{1-2\alpha}}{\beta_1 (1-\alpha)^2 }
  \le
  \frac{4 K  }{\beta_1 (1-\alpha) }
  \le
  1
  \le
  \frac{1-\alpha}{4} \beta_1
  \le
  (1 - \sqrt{2}^{\alpha-1}) \beta_t
\end{align}
and hence the condition \eqref{eq:boundqr1condibeta} in Lemma~\ref{lem:boundqr1} holds.
If $\alpha > 1/2$,
as we have $\alphab = 1-\alpha$,
from the definition of $\betab$ in \eqref{eq:paramMAB2},
we obtain
\begin{align}
  \frac{4 \alpha K q_{t*}^{1-2\alpha}}{\beta_1 (1-\alpha)^2 }
  \le
  \frac{\alpha}{8}
  \betab
  q_{t*}^{1-2\alpha} 
  =
  \frac{\alpha}{8}
  \betab
  q_{t*}^{\alphab-\alpha} 
  \le
  \frac{1-\sqrt{2}^{\alphab-1}}{\sqrt{2}} 
  \betab
  q_{t*}^{\alphab-\alpha} 
  ,
\end{align}
which implies the condition \eqref{eq:boundqr1condibeta} in Lemma~\ref{lem:boundqr1} holds.
Hence,
by applying Lemma~\ref{lem:boundqr1}
with $\ell = 0$,
$\beta = \beta_t$,
$\beta' = \beta_{t+1}$,
and $\alphab =  1-\alpha$,
we obtain
$q'_{ti} \le 2q_{ti}$ for all $i \in [K]$.
Further,
as we have
$\beta_{t+1} \ge \beta_1 \ge
\frac{4 K}{1-\alpha}
\ge
\frac{2 K}{1 - 2^{\alpha - 1}} $,
we can apply Lemma~\ref{lem:boundqr2} with
$\omega = 2$,
$E = \{ (i, i) \mid i \in [K] \}$,
$\ell = \hat{\ell}_t$,
and
$\beta = \beta_{t+1} / 2$
to obtain 
$q_{t+1,i} \le 2q'_{ti}$ for all $i \in [K]$.
We hence have
$q_{t+1, i} \le 4 q_{ti}$ for all $i \in [K]$.
Therefore,
from Lemma~\ref{lem:psiqr},
we obtain
$h_{t+1} = O(h_t)$,
which means that the first part of \eqref{eq:BOBWcondition} holds.

\paragraph{Verifying condition \eqref{eq:htztomega}}
For any $i^* \in [K]$,
we have
\begin{align}
		\nonumber
	z_t
	&
	=
	\frac{1}{1-\alpha}
		\sum_{i=1}^{K} \tilde{q}_{ti}^{1-\alpha}
	\le
	\frac{1}{1-\alpha}
	\left(
		\sum_{i \in [K] \setminus \{ \tilde{I}(t) \}} {q}_{ti}^{1-\alpha}
		+
		(1 - q_{t, \tilde{I}(t)})^{1-\alpha}
	\right)
	\\
	&
	\le
	\frac{2}{1-\alpha}
	\left(
		\sum_{i \in [K] \setminus \{ \tilde{I}(t) \}} {q}_{ti}^{1-\alpha}
	\right)
	\le
	\frac{2}{1-\alpha}
	\left(
		\sum_{i \in [K] \setminus \{ i^* \}} {q}_{ti}^{1-\alpha}
	\right)
	\le
	\frac{2 (K-1)^{\alpha}}{1-\alpha}
\end{align}
and
\begin{align}
	h_t
	&
	=
	\frac{1}{\alpha}
	\left(
		\sum_{i=1}^{K} {q}_{ti}^{\alpha}
		-1
	\right)
	\le
	\frac{1}{\alpha}
	\left(
		\sum_{i=1}^{K} {q}_{ti}^{\alpha}
		- q_{t,i^*}^{\alpha}
	\right)
	=
	\frac{1}{\alpha}
		\sum_{i \in [K] \setminus \{ i^* \}} {q}_{ti}^{\alpha}
	\le
	\frac{(K-1)^{1-\alpha}}{\alpha}.
	\label{eq:htbound1}
\end{align}
We hence have $h_1 z_{\max} \le \frac{2(K-1)}{(1-\alpha) \alpha}$.
Further,
from H\"older's inequality,
we have
\begin{align}
		\nonumber
	z_t
	&
	\le
	\frac{2}{1-\alpha}
	\left(
		\sum_{i \in [K] \setminus \{ i^* \}} {q}_{ti}^{1-\alpha}
	\right)
	=
	\frac{2}{1-\alpha}
	\left(
		\sum_{i \in [K] \setminus \{ i^* \}} \frac{1}{\Delta_i^{1-\alpha}} (\Delta_{i} {q}_{ti})^{1-\alpha}
	\right)
	\\
	&
	\le
	\frac{2}{1-\alpha}
	\left(
		\sum_{i \in [K] \setminus \{ i^* \}} 
		\frac{1}{\Delta_i^{\frac{1-\alpha}{\alpha}}} 
	\right)^{\alpha}
	\left(
		\sum_{i \in [K] \setminus \{ i^* \}} 
		\Delta_{i} {q}_{ti}
	\right)^{1-\alpha}
\end{align}
and 
\begin{align}
		\nonumber
	h_t
	&
	\le
	\frac{1}{\alpha}
	\left(
		\sum_{i \in [K] \setminus \{ i^* \}} {q}_{ti}^{\alpha}
	\right)
	=
	\frac{1}{\alpha}
	\left(
		\sum_{i \in [K] \setminus \{ i^* \}} \frac{1}{\Delta_i^{\alpha}} (\Delta_{i} {q}_{ti})^{\alpha}
	\right)
	\\
	&
	\le
	\frac{1}{\alpha}
	\left(
		\sum_{i \in [K] \setminus \{ i^* \}} 
		\frac{1}{\Delta_i^{\frac{\alpha}{1-\alpha}}} 
	\right)^{1-\alpha}
	\left(
		\sum_{i \in [K] \setminus \{ i^* \}} 
		\Delta_{i} {q}_{ti}
	\right)^{\alpha}.
	\label{eq:htbound2}
\end{align}
We hence have
\begin{align}
	h_t
	z_t 
	\le
	\frac{2}{\alpha (1-\alpha)} 
	\left(
	\sum_{i \in [K] \setminus \{i^*\}} \Delta_i^{-\frac{ \alpha }{1-\alpha}}
	\right)^{1-\alpha}
	\left(
	\sum_{i \neq [K] \setminus \{i^*\}} \Delta_i^{-\frac{1- \alpha }{\alpha}}
	\right)^{\alpha}
	\linner \Delta, q_t \rinner,
\end{align}
which means that \eqref{eq:htztomega} holds with $\omega(\Delta)$ defined by \eqref{eq:defomegaMAB}.

\subsection{Linear Bandit: Proof of Theorem~\ref{thm:BOBWLinear}}
From Proposition~\ref{lem:BOBW},
it suffices to verify that conditions \eqref{eq:BOBWcondition} and \eqref{eq:htztomega} hold.
\paragraph{Verifying condition \eqref{eq:BOBWcondition}}
From~\eqref{eq:paramLinear} and \eqref{eq:defp0linear},
we have
\begin{align}
  |\hat{\ell}_{ti}|
  \le
  \frac{cd}{\gamma_t}
  \le
  \frac{(1-\alpha) q_{t*}^{\alpha-1}}{4}\beta_t.
\end{align}
Hence,
we can apply Lemma~\ref{lem:Tsallis-stab2} to obtain the following:
\begin{align}
		\nonumber
	&
	\E \left[
		\linner 
		\hat{\ell}_t,
		q_{t} - q_{t+1}
		\rinner
		-
		\beta_{t} D(q_{t+1}, q_t)
		|
		\cH_{t-1}
	\right]
	\le
	\frac{4}{(1-\alpha) \beta_t}
	\E \left[
		\sum_{i = 1}^K
		\hat{\ell}_{ti}^2 \tilde{q}_{ti}^{2-\alpha}
		|
		\cH_{t-1}
	\right]
	\\
		\nonumber
	&
	\le
	\frac{4}{(1-\alpha) \beta_t}
	\E \left[
		\sum_{i = 1}^K
		\phi_i^\top
		S(p_t)^{-1} 
		\phi_{I(t)}
		\phi_{I(t)}^\top
		S(p_t)^{-1} 
		\phi_i 
		\tilde{q}_{ti}^{2-\alpha}
		|
		\cH_{t-1}
	\right]
	\\
		\nonumber
	&
	=
	\frac{4}{(1-\alpha) \beta_t}
	\sum_{i = 1}^K
	\phi_i^\top
	S(q_t)^{-1} 
	\phi_i 
	\tilde{q}_{ti}^{2-\alpha}
	\le
	\frac{8}{(1-\alpha) \beta_t}
	\sum_{i = 1}^K
	\phi_i^\top
	S(q_t)^{-1} 
	\phi_i 
	\tilde{q}_{ti}^{2-\alpha}
	\\
		\nonumber
	&
	=
	\frac{8}{(1-\alpha) \beta_t}
	\mathrm{tr}
	\left(
	S(q_t)^{-1} 
	\sum_{i = 1}^K
	\phi_i 
	\phi_i^\top
	\tilde{q}_{ti}^{2-\alpha}
	\right)
	\le
	\frac{8}{(1-\alpha) \beta_t}
	\mathrm{tr}
	\left(
	S(q_t)^{-1} 
	\sum_{i = 1}^K
	\phi_i 
	\phi_i^\top
	{q}_{ti}
	\right)
	{q}_{t*}^{1-\alpha}
	\\
	&
	=
	\frac{8}{(1-\alpha) \beta_t}
	\mathrm{tr}(I_d)
	{q}_{t*}^{1-\alpha}
	=
	\frac{8d}{(1-\alpha) \beta_t}
	{q}_{t*}^{1-\alpha}
	=
	O \left(
		\frac{z_t}{\beta_t}
	\right),
\end{align}
where $\mathrm{tr}(M)$ represents the trace of a matrix $M$ and
$I_d \in \re^{d \times d}$ denotes the identity matrix of size $d$.
As it is clear from the definition of $\gamma_t$ in \eqref{eq:paramLinear} that $\gamma_t = O(z_t/\beta_t)$,
we can verify that the second part of \eqref{eq:BOBWcondition} holds.
We next see that $h_{t+1} = O(h_t)$.
From \eqref{eq:LBht} and the definition of $z_t$,
we have
\begin{align}
	\beta_{t+1} - \beta_t
	=
	\frac{z_t}{\beta_t \hat{h}_{t+1}}
	=
	\frac{z_t}{\beta_t {h}_{t}}
	\le
	\frac{4  \alpha d q_{t*}^{1-2\alpha}}{\beta_1 (1-\alpha)^2}
	\le
	\frac{\betab \alpha q_{t*}^{1-2\alpha}}{8},
\end{align}
where the first equality comes from \eqref{eq:stability-agnostic},
the second equality follows from the definition $\hat{h}_t$ in Algorithm~\ref{alg:FTRL},
the first inequality follows from \eqref{eq:LBht} and the definition of $z_t$ in \eqref{eq:paramLinear},
and the last inequality follows from the condition on $\betab$ in \eqref{eq:paramLinear}.
Thus,
we can apply Lemma~\ref{lem:boundqr1} to $\ell = \hat{\ell}_t$,
$\beta = \beta_t$,
and
$\beta' = \beta_{t+1}$
to obtain $h_{t+1} = O(h_t)$.
Therefore,
it has been confirmed that condition \eqref{eq:BOBWcondition} is satisfied.

\paragraph{Verifying condition \eqref{eq:htztomega}}
From the definition of $z_t$ in \eqref{eq:paramLinear},
and from \eqref{eq:htbound1},
we have
$h_1 z_{\max} \le \frac{d}{\alpha(1-\alpha)} K^{1-\alpha} $.
In addition,
for any $i^* \in [K]$
we have
\begin{align}
		\nonumber
	z_t 
	&
	\le \frac{d}{1-\alpha} \left(1 - q_{t, \tilde{I}(t)} \right)^{1-\alpha} 
	\le \frac{d}{1-\alpha} \left(1 - q_{t, i^*} \right)^{1-\alpha}
	\\
	&
	\le 
	\frac{d}{(1-\alpha)\Delta_{\min}^{1-\alpha}} 
	\left( \Delta_{\min} \sum_{i \in [K] \setminus \{ i^* \}} q_{ti} \right)^{1-\alpha}
	\le 
	\frac{d}{(1-\alpha)\Delta_{\min}^{1-\alpha}} 
	\left( \linner \Delta, q_t \rinner \right)^{1-\alpha}.
\end{align}
By combining this with \eqref{eq:htbound2},
we obtain
\begin{align}
	h_t z_t
	\le
	\frac{d}{\alpha (1-\alpha)} 
	\Delta_{\min}^{\alpha-1}
	\left(
	\sum_{i \neq i^*} \Delta_i^{-\frac{\alpha }{1-\alpha}}
	\right)^{1-\alpha}
	\linner \Delta, q_t \rinner,
\end{align}
which implies that \eqref{eq:htztomega} holds with $\omega(\Delta)$ defined by \eqref{eq:defomegaLinear}.


\subsection{Graph bandit}
\label{sec:graph}
In the \textit{graph bandit} problems,
the player is given \textit{feedback graph} $G = (V, E)$,
where $V = [K]$ is the set of vertices and $E \subseteq V \times V$ is the set of edges.
In this paper,
we assume that the graph is undirected and
that every vertex has a self-loop,
i.e.,
$(i, j) \in E$ if $(j, i) \in E$
and
$(i, i) \in E$ for all $i,j \in V$.
Denote $N(i) = \{ j \in [K] \mid (i, j) \in E \}$.
The feedback from the environment is the values of losses for vertices adjacent to the chosen vertex,
i.e.,
the player can observe $\ell_{t i}$ for all $i \in N(I(t))$,
after incurring the loss of $\ell_{t, I(t)}$.
Let
$P_{ti} \in [0,1]$
denote the probability that $\ell_{ti}$ is observed,
i.e.,
let
$P_{ti} = \sum_{j \in N(i)} p_{tj}$.
Let $\zeta \ge 1$ denote the independence number of the feedback graph $G$.

In applying Algorithm~\ref{alg:FTRL} to graph bandit problems,
we choose arbitrary $\alpha \in (0, 1)$ and
set parameters as
\begin{align}
  \label{eq:paramGraph}
  \beta_1 \ge
    \frac{4K}{1-\alpha},
  \quad
  z_t = \frac{1}{1-\alpha} \sum_{i =1}^{K} \frac{\tilde{q}_{ti}^{2-\alpha}}{P_{ti}},
  \quad
	\gamma_t = 0,
	\quad
	\hat{\ell}_{ti}
	=
	\frac{\mathbf{1}[i \in N(I(t))]}{P_{ti}} \ell_{ti}.
\end{align}
We also set $\betab \ge 0$ 
by \eqref{eq:paramMAB2}.
Then,
\eqref{eq:BOBWcondition} holds
under the conditions of
\eqref{eq:paramMAB2} and \eqref{eq:paramGraph}.
In addition,
we can show that
$h_t = - \psi(q_t)$
and
$z_t $ in \eqref{eq:paramMAB}
satisfy
$h_1 z_t \le 2 \frac{\zeta}{\alpha (1-\alpha)} \left( \frac{K}{\zeta} \right)^{1-\alpha}$ and
that
\eqref{eq:htztomega} holds with
$\omega(\Delta)$ defined by
\begin{align}
  \label{eq:defomegaGraph}
  \omega(\Delta)
  =
  \frac{2 \zeta^{\alpha}}{\alpha (1-\alpha)} 
  \Delta_{\min}^{\alpha-1}
  \left(
    \sum_{i \neq i^*} \Delta_i^{-\frac{\alpha }{1-\alpha}}
  \right)^{1-\alpha}
  \le
  \frac{2 \zeta}{\alpha (1-\alpha) \Delta_{\min}} 
  \left(
    \frac{K}{\zeta}
  \right)^{1-\alpha}
  .
\end{align}
Hence,
Proposition~\ref{lem:BOBW} leads to the following regret bounds:
\begin{theorem}
  \label{thm:BOBWGraph}
  Let $G = (V=[K], E)$ be an undirected graph,
  of which all vertices have self-loops,
  with the independence number $\zeta \ge 1$.
  For the graph bandit problem associated with $G$,
  Algorithm~\ref{alg:FTRL} with \eqref{eq:paramGraph} and \eqref{eq:paramMAB2}
  achieves BOBW regret bounds in Proposition~\ref{lem:BOBW} with 
  $h_1 z_{\max} = O\left(
	\frac{\zeta}{\alpha (1-\alpha)} \left( \frac{K}{\zeta} \right)^{1-\alpha}
\right)$
  and
  $\omega(\Delta)$ given by \eqref{eq:defomegaGraph}.
\end{theorem}
\begin{proof}
From Proposition~\ref{lem:BOBW},
it suffices to verify that conditions \eqref{eq:BOBWcondition} and \eqref{eq:htztomega} hold.
\paragraph{Verifying condition \eqref{eq:BOBWcondition}}
As $\hat{\ell}_{t, \tilde{I}(t)} \le \frac{\ell_{t,\tilde{I}(t)}}{p_{t, \tilde{I}(t)}} \le K$
we can apply Lemma~\ref{lem:Tsallis-stab} with $i^* = \tilde{I}(t)$ to obtain
\begin{align}
		\nonumber
	&
	\E \left[
		\linner 
		\hat{\ell}_t,
		q_{t} - q_{t+1}
		\rinner
		-
		\beta_{t} D(q_{t+1}, q_t)
		|
		\cH_{t-1}
	\right]
	\le
	\frac{4}{(1-\alpha) \beta_t}
	\E \left[
		\sum_{i = 1 }^K
		\hat{\ell}_{ti}^2 \tilde{q}_{ti}^{2-\alpha}
		|
		\cH_{t-1}
	\right]
	\\
	&
	\le
	\frac{4}{(1-\alpha) \beta_t}
	\E \left[
		\sum_{i = 1}^K
		\frac{ \mathbf{1}\{I_t \in N(i)\} }{P_{ti}^2}
		\tilde{q}_{ti}^{2-\alpha}
		|
		\cH_{t-1}
	\right]
	=
	\frac{4}{(1-\alpha) \beta_t}
	\sum_{i = 1}^K
	\frac{ 
	\tilde{q}_{ti}^{2-\alpha}
	}{P_{ti}}
	=
	O \left(
		\frac{z_t}{\beta_t}
	\right).
\end{align}
Further,
$h_{t+1} = O(h_t)$ can be shown following the approach outlined in Section~\ref{sec:ProofMAB}.
Thus,
it has been confirmed that condition \eqref{eq:BOBWcondition} is satisfied.

\paragraph{Verifying condition \eqref{eq:htztomega}}
We can obtain a bound on $z_t$
from Lemma~1 by \citet{eldowa2023minimax} as follows:
\begin{lemma}
  \label{lem:boundztGraph}
  Let $\zeta \ge 1$ be the independence number of $G$.
  We then have
  \begin{align}
    \label{eq:boundztGraph}
    \sum_{i =1}^K
    \frac{\tilde{q}_{ti}^{2-\alpha}}{P_{ti}}
    \le
    (1 + \zeta^{\alpha})
    \left( 1 - q_{t, \tilde{I}(t)} \right)^{1-\alpha}.
  \end{align}
\end{lemma}
\begin{proof}
  From the proof of Lemma~1 by \citet{eldowa2023minimax},
  there exists an independent set $S \subseteq [K] \setminus \{ \tilde{I}(t) \}$ such that
  \begin{align}
    \sum_{i \in [K] \setminus \{ \tilde{I}(t) \}}
    \frac{q_{ti}^{2-\alpha}}{P_{ti}}
    \le
    \sum_{i \in S}
    q_{ti}^{1-\alpha}.
  \end{align}
  From H\"older's inequality,
  we have
  \begin{align}
    \sum_{i \in S}
    q_{ti}^{1-\alpha}
    \le
    |S|^\alpha
    \left(
    \sum_{i \in S}
    q_{ti}
    \right)^{1-\alpha}.
  \end{align}
  As $S$ is an independent set of $G$ and
  $S \subseteq [K] \setminus \{ \tilde{I}(t) \}$,
  we have
  \begin{align}
    |S|^\alpha
    \left(
    \sum_{i \in S}
    q_{ti}
    \right)^{1-\alpha}
    \le
    \zeta^{\alpha}
    \left(
    \sum_{i \in [K] \setminus \{\tilde{I}(t)\}}
    q_{ti}
    \right)^{1-\alpha}
    =
    \zeta^{\alpha}
    \left(
      1 - q_{t, \tilde{I}(t)}
    \right)^{1-\alpha}.
  \end{align}
  In addition,
  we have
  \begin{align}
    \frac{
    q_{t*}^{2-\alpha}
    }{P_{t,\tilde{I}(t)}}
    \le
    \frac{
    q_{t*}^{2-\alpha}
    }{q_{t*}}
    \le
    q_{t*}^{1-\alpha}
    \le
    \left(1 - q_{t, \tilde{I}(t)} \right)^{1-\alpha}.
  \end{align}
  Combining these inequalities,
  we obtain \eqref{eq:boundztGraph}.
\end{proof}
From this lemma,
we have
\begin{align}
  z_t \le \frac{1 + \zeta^{\alpha}}{1-\alpha} \left(1 - q_{t, \tilde{I}(t)} \right)^{1-\alpha} .
\end{align}
From this and \eqref{eq:htbound1},
we have
$h_1 z_{\max} \le \frac{2}{\alpha(1-\alpha)} \zeta^{\alpha} K^{1-\alpha} $.
In addition,
for any $i^* \in [K]$
we have
\begin{align}
		\nonumber
	z_t 
	&
	\le \frac{1 + \zeta^{\alpha}}{1-\alpha} \left(1 - q_{t, \tilde{I}(t)} \right)^{1-\alpha} 
	\le \frac{1 + \zeta^{\alpha}}{1-\alpha} \left(1 - q_{t, i^*} \right)^{1-\alpha}
	\\
	&
	\le 
	\frac{1 + \zeta^{\alpha}}{(1-\alpha)\Delta_{\min}^{1-\alpha}} 
	\left( \Delta_{\min} \sum_{i \in [K] \setminus \{ i^* \}} q_{ti} \right)^{1-\alpha}
	\le 
	\frac{1 + \zeta^{\alpha}}{(1-\alpha)\Delta_{\min}^{1-\alpha}} 
	\left( \linner \Delta, q_t \rinner \right)^{1-\alpha}.
\end{align}
By combining this with \eqref{eq:htbound2},
we obtain
\begin{align}
	h_t z_t
	\le
	\frac{2 \zeta^{\alpha}}{\alpha (1-\alpha)} 
	\Delta_{\min}^{\alpha-1}
	\left(
	\sum_{i \neq i^*} \Delta_i^{-\frac{\alpha }{1-\alpha}}
	\right)^{1-\alpha}
	\linner \Delta, q_t \rinner,
\end{align}
which implies that \eqref{eq:htztomega} holds with $\omega(\Delta)$ defined by \eqref{eq:defomegaGraph}.
\end{proof}

Note that we can obtain
$
\frac{\zeta}{\alpha(1-\alpha)}\left( \frac{K}{\zeta} \right)^{1-\alpha} 
= O \left(  \zeta \log \left(1 + \frac{K}{\zeta} \right) \right)$
by setting
$\alpha = 1 - \frac{1}{2 (1 + \log (K / \zeta))}$,
which recovers the minimax regret upper bound shown by \citet{eldowa2023minimax}.

\subsection{Contextual bandit}
\label{sec:contextual}
In the \textit{contextual bandit} problems,
or the bandit problems with expert advices,
each action $i$ is associated with an \textit{expert},
which provides an \textit{advice} $ \phi_{ti} \in \cP(M) $ in each round $t$.
After choosing an expert $I(t) \in [K]$,
the player can observe the advices $\phi_{ti}$ of all experts $i \in [K]$,
and pick $J(t) \in [M]$ following the distribution of $\phi_{t, I(t)}$.
Then the player gets feedback of the incurred loss $\ell'_{t, J(t)}$,
where $\ell'_t \in [0,1]^M$ is chosen by the environment before the player chooses $I(t)$.
Let 
$
  P_{t} \in \cP(M)
$
denote the distribution that $J(t)$ follows
given $p_t$ and $\left\{ \phi_{ti} \right\}_{i=1}^K$,
i.e.,
$
  P_{tj}
  =
  \sum_{i=1}^K
  p_{ti}
  \phi_{tij}
$.

Let $\alpha \ge 1/2$ and set
\begin{align}
  \label{eq:paramContextual}
  \beta_1 \ge
  \frac{8 K}{1-\alpha},
  ~
  \betab \ge 
  \frac{32 M}{(1-\alpha)^2 \beta_1},
  ~
  z_t = \frac{M q_{t*}^{1-\alpha}}{1-\alpha} ,
  ~
	\gamma_t = 0,
  ~
	\hat{\ell}_{ti}
	=
  \frac{
  \ell'_{t, J(t)}
  \phi_{ti, J(t)}}{ P_{t, J(t)} }.
\end{align}
If parameters are given by
\eqref{eq:paramContextual},
then
\eqref{eq:BOBWcondition} holds.
Further,
$h_t = - \psi(q_t)$
and
$z_t $ in \eqref{eq:paramContextual}
satisfy
$h_1 z_t \le \frac{M}{\alpha (1-\alpha)} K^{1-\alpha}$ and
\eqref{eq:htztomega} with
$\omega(\Delta)$ defined as
\begin{align}
  \label{eq:defomegaContextual}
  \omega(\Delta)
  =
  \frac{M}{\alpha (1-\alpha)} 
  \Delta_{\min}^{\alpha-1}
  \left(
    \sum_{i \neq i^*} \Delta_i^{-\frac{\alpha }{1-\alpha}}
  \right)^{1-\alpha}
  \le
  \frac{MK^{1-\alpha}}{\alpha (1-\alpha) \Delta_{\min}} 
  .
\end{align}
Hence,
Proposition~\ref{lem:BOBW} leads to the following regret bounds:
\begin{theorem}
  \label{thm:BOBWContextual}
  For contextual bandit problems of $M$ arms with $K$ experts,
  Algorithm~\ref{alg:FTRL} with 
  parameters given by \eqref{eq:paramContextual}
  achieves BOBW regret bounds in Proposition~\ref{lem:BOBW} with 
  $h_1 z_{\max} = O\left(
	\frac{M K^{1-\alpha}}{\alpha (1-\alpha)} 
\right)$
  and
  $\omega(\Delta)$ given by \eqref{eq:defomegaContextual}.
\end{theorem}
\begin{proof}
From Proposition~\ref{lem:BOBW},
it suffices to verify that conditions \eqref{eq:BOBWcondition} and \eqref{eq:htztomega} hold.
\paragraph{Verifying condition \eqref{eq:BOBWcondition}}
As $q_{t,\tilde{I}(t)} \ge 1/K$,
we have
\begin{align}
	\label{eq:ellhattildeI}
	\hat{\ell}_{t, \tilde{I}(t)}
	=
	\frac{
	\ell'_{t, J(t)}
	\phi_{t,\tilde{I}(t), J(t)}}{ P_{t, J(t)} }
	\le
	\frac{\phi_{t,\tilde{I}(t),J(t)}}{\sum_{i=1}^K q_{ti} \phi_{ti,J(t)}}
	\le
	\frac{\phi_{t,\tilde{I}(t),J(t)}}{q_{t,\tilde{I}(t)} \phi_{t,\tilde{I}(t),J(t)}}
	\le
	\frac{1}{q_{t, \tilde{I}(t)}}
	\le
	K.
\end{align}
Hence,
we can apply Lemma~\ref{lem:Tsallis-stab} with $i^* = \tilde{I}(t)$ to obtain
\begin{align}
		\nonumber
	&
	\E \left[
		\linner 
		\hat{\ell}_t,
		q_{t} - q_{t+1}
		\rinner
		-
		\beta_{t} D(q_{t+1}, q_t)
		|
		\cH_{t-1}
	\right]
	\le
	\frac{4}{(1-\alpha) \beta_t}
	\E \left[
		\sum_{i = 1 }^K
		\hat{\ell}_{ti}^2 \tilde{q}_{ti}^{2-\alpha}
		|
		\cH_{t-1}
	\right]
	\\
		\nonumber
	&
	\le
	\frac{4}{(1-\alpha) \beta_t}
	\E \left[
		\sum_{i = 1}^K
		\frac{ \phi_{ti,J(t)}^2 }{P_{t,J(t)}^2}
		\tilde{q}_{ti}^{2-\alpha}
		|
		\cH_{t-1}
	\right]
	\le
	\frac{4}{(1-\alpha) \beta_t}
	\E \left[
		\sum_{i = 1}^K
		\frac{ \phi_{ti,J(t)} }{P_{t,J(t)}^2}
		\tilde{q}_{ti}^{2-\alpha}
		|
		\cH_{t-1}
	\right]
	\\
		\nonumber
	&
	=
	\frac{4 }{(1-\alpha) \beta_t}
	\sum_{j=1}^M
	\sum_{i = 1}^K
	\frac{ \phi_{tij} }{P_{tj}}
	\tilde{q}_{ti}^{2-\alpha}
	\le
	\frac{4 }{(1-\alpha) \beta_t}
	\sum_{j=1}^M
	\sum_{i = 1}^K
	\frac{ q_{ti} \phi_{tij} }{P_{tj}}
	{q}_{t*}^{1-\alpha}
	\\
	&
	=
	\frac{4 }{(1-\alpha) \beta_t}
	\sum_{j=1}^M
	\frac{ P_{tj} }{P_{tj}}
	{q}_{t*}^{1-\alpha}
	=
	\frac{4 M q_{t*}^{1-\alpha}}{(1-\alpha) \beta_t}
	=
	O
	\left(
		\frac{z_t}{\beta_t}
	\right).
\end{align}
Further,
$h_{t+1} = O(h_t)$ can be shown following the approach outlined in Section~\ref{sec:ProofMAB}.
In fact,
as we have
$\hat{\ell}_{ti} \ge 0$ and $\sum_{i = 1}^K q_{ti} \hat{\ell}_{ti} \le 1$
and $\beta_t \ge \frac{8K}{1-\alpha}$,
we can apply Lemma~\ref{lem:boundqr3} with $\omega = 2$, $\ell = \hat{\ell}_t$.
In addition,
\eqref{eq:ellhattildeI} and the definition of $\beta$ and $\betab$ in \eqref{eq:paramContextual} ensure that
we can apply Lemma~\ref{lem:boundqr1} with $\omega = 2$, $\ell = 0$, and $i^* = \tilde{I}(t)$.
Thus,
it has been confirmed that condition \eqref{eq:BOBWcondition} is satisfied.

\paragraph{Verifying condition \eqref{eq:htztomega}}
From the definition of $z_t$ in \eqref{eq:paramContextual},
and from \eqref{eq:htbound1},
we have
$h_1 z_{\max} \le \frac{d}{\alpha(1-\alpha)} K^{1-\alpha} $.
In addition,
for any $i^* \in [K]$
we have
\begin{align*}
	z_t 
	&
	\le \frac{M}{1-\alpha} \left(1 - q_{t, \tilde{I}(t)} \right)^{1-\alpha} 
	\le \frac{M}{1-\alpha} \left(1 - q_{t, i^*} \right)^{1-\alpha}
	\\
	&
	\le 
	\frac{M}{(1-\alpha)\Delta_{\min}^{1-\alpha}} 
	\left( \Delta_{\min} \sum_{i \in [K] \setminus \{ i^* \}} q_{ti} \right)^{1-\alpha}
	\le 
	\frac{M}{(1-\alpha)\Delta_{\min}^{1-\alpha}} 
	\left( \linner \Delta, q_t \rinner \right)^{1-\alpha}.
\end{align*}
By combining this with \eqref{eq:htbound2},
we obtain
\begin{align*}
	h_t z_t
	\le
	\frac{M}{\alpha (1-\alpha)} 
	\Delta_{\min}^{\alpha-1}
	\left(
	\sum_{i \neq i^*} \Delta_i^{-\frac{\alpha }{1-\alpha}}
	\right)^{1-\alpha}
	\linner \Delta, q_t \rinner,
\end{align*}
which implies that \eqref{eq:htztomega} holds with $\omega(\Delta)$ defined by \eqref{eq:defomegaContextual}.
\end{proof}

Note that we
obtain
$\frac{MK^{1-\alpha}}{\alpha (1-\alpha)} = O (M \log K)$
by setting
$\alpha = 1 - \frac{1}{4 \log K}$,
which recovers the regret upper bound by \citet[Corollary 13]{dann2023blackbox}.

\end{document}